\documentclass[10pt]{article} 

\usepackage[accepted]{rlj} 

\usepackage{amssymb}            
\usepackage{mathtools}          
\usepackage{mathrsfs}           
\usepackage{graphicx}           
\usepackage{subcaption}         
\usepackage[space]{grffile}     
\usepackage{url}                
\usepackage{lipsum}             

\usepackage{amsfonts}
\usepackage{amsthm}
\usepackage[capitalize,noabbrev]{cleveref}
\usepackage{tikz}
\usetikzlibrary{positioning}
\usepackage{bm}
\usepackage{gmverse}
\usepackage{booktabs} 

\usepackage{algorithm}
\usepackage{algpseudocode}

\newtheorem{proposition}{Proposition}[section]
\newtheorem{theorem}{Theorem}[section]
\newtheorem{example}{Example}[section]
\newtheorem{definition}{Definition}[section]
\newtheorem{remark}{Remark}[section]

\tikzstyle{vertex}=[draw,circle,minimum size=25pt,inner sep=0pt]
\tikzstyle{selected vertex} = [vertex, fill=red!24]
\tikzstyle{edge} = [draw,thick,->]
\tikzstyle{weight} = [font=\normalsize]

\usepackage{wrapfig}
\newcommand{\St}{{\mathcal S}}
\newcommand{\Ac}{{\mathcal A}}

\newcommand{\Md}{{\mathcal M}}

\newcommand{\R}{{\bm r}}
\newcommand{\J}{J} 
\newcommand{\w}{{\bm \omega}} 
\newcommand{\nO}{N} 
\newcommand{\swf}{\phi_{\w}} 

\usepackage{amsmath}
\DeclareMathOperator*{\argmax}{argmax}
\newcommand{\Expect}{\mathbb E}

\newcommand{\KL}{\mathrm{KL}}
\newcommand{\sD}{\mathcal{D}}
\newcommand{\tS}{\widetilde{\mathcal{S}}}
\newcommand{\tM}{\widetilde{\mathcal{M}}}

\title{Inference-Time Policy Alignment for Fair Reinforcement Learning}

\setrunningtitle{Policy Shaping for Fair RL}

\author{Umer Siddique, Peilang Li, Conor Wallace, Yongcan Cao}

\emails{\{muhammadumer.siddique,peilang.li,conor.wallace\}@my.utsa.edu, yongcan.cao@utsa.edu}

\affiliations{
\textbf{Department of Electrical and Computer Engineering, \\University of Texas at San Antonio}\\
}

\contribution{
    We formalize \emph{inference-time fairness alignment} in RL as a policy shaping problem, enabling a frozen reward-maximizing policy to be steered toward welfare-based fairness objectives at deployment without retraining.
    }
    {
    While prior work on welfare-based fairness in RL~\citep{SiddiqueWengZimmer20,fan2022welfare,YuSiddiqueWeng23} strictly requires training-time optimization, and inference-time alignment techniques~\citep{lu2023inference,mujtaba2025aligning} have largely focused on autoregressive generation, we bridge this gap and achieve inference-time fairness alignment in sequential decision-making without retraining the base policy.
    }

\contribution{
    We construct a welfare-augmented MDP that restores Markovianity for non-linear welfare functions and derive, in closed form, the shaped policy that maximizes a KL-regularized welfare surrogate, an exponential reweighting of the base policy which provides a local lower bound on the welfare improvement under the KL constraint.
    }
    {
    Although the KL-regularized policy improvement is a foundational result in RL~\citep{SchulmanLevineAbbeelJordanMoritz15}, its application to non-linear, trajectory-level welfare functions is previously unexplored. We provide a formal theoretical guarantee that the state augmentation forms a valid augmented MDP and show that welfare-based fairness can be optimized through this structure by introducing an augmented state that tracks cumulative rewards.
    }

\contribution{
    We propose QFair, a lightweight welfare critic trained offline on base-policy rollouts, which evaluates actions based on both the current state and the evolving cumulative reward vector.
    }
    {
    By proving a trajectory equivalence result between the original and augmented MDPs, we demonstrate that QFair can be learned from base-policy rollouts with light exploration, requiring no welfare-driven environment interaction. This avoids the computationally expensive or unsafe welfare-oriented training interactions required by standard multi-objective fair RL methods~\citep{SiddiqueWengZimmer20,zimmer2021learning,ju2023achieving,kim2025fairdice}.
    }

\keywords{Reinforcement learning, Inference-time alignment, Policy shaping, Fair optimization} 

\summary{Deep reinforcement learning (RL) agents achieve strong performance by optimizing scalar reward functions. However, once deployed, the policies of these RL agents are often rigid and costly to adapt to new performance criteria. For instance, an agent trained to maximize expected cumulative reward may not accommodate previously unknown stakeholder preferences. Existing approaches to achieve fairness, a type of preference, in RL typically assume that such preferences are known \textit{a priori} and require complete retraining of the policy under a fairness-oriented metric. Inspired by inference-time alignment in large language models, we investigate the problem of steering a pretrained RL policy toward welfare-based fairness objectives at inference time without updating the base policy's parameters. We formalize inference-time fairness alignment as a policy shaping problem and propose a multiplicative policy shaping framework that adjusts action probabilities using action-dependent welfare scores, thus requiring no modification to the base policy. Our framework is general and compatible with any deep RL agent. Through extensive experiments across multiple domains, we demonstrate that inference-time policy shaping substantially improves welfare-based fairness objectives while preserving core task performance.
}

\begin{document}

\maketitle  

\begin{abstract}
Deep reinforcement learning (RL) agents achieve strong performance by optimizing scalar reward functions. However, once deployed, the policies of these RL agents are often rigid and costly to adapt to new performance criteria. For instance, an agent trained to maximize expected cumulative reward may not accommodate previously unknown stakeholder preferences. Existing approaches to achieve fairness, a type of preference, in RL typically assume that such preferences are known \textit{a priori} and require complete retraining of the policy under a fairness-oriented metric. Inspired by inference-time alignment in large language models, we investigate the problem of steering a pretrained RL policy toward welfare-based fairness objectives at inference time without updating the base policy's parameters. We formalize inference-time fairness alignment as a policy shaping problem and propose a multiplicative policy shaping framework that adjusts action probabilities using action-dependent welfare scores, thus requiring no modification to the base policy. Our framework is general and compatible with any deep RL agent. Through extensive experiments across multiple domains, we demonstrate that inference-time policy shaping substantially improves welfare-based fairness objectives while preserving core task performance.
\end{abstract}

\section{Introduction}
\label{sec:intro}
Deep RL has demonstrated remarkable success in domains ranging from robotics~\citep{peters2003reinforcement} to game playing~\citep{silver2017mastering}, typically by optimizing for a single scalar reward function~\citep{SuttonBarto98}. 
These achievements are often built on fixed task definitions and carefully engineered training pipelines, where RL agents learn policies that maximize return by interacting with an environment.
However, such policies can be rigid and brittle at deployment and hence, incorporating new constraints, accommodating evolving stakeholder preferences, or handling out-of-distribution (OOD) shifts \citep{haider2023out,ajay2022distributionally} typically requires costly retraining or delicate fine-tuning, with no guarantee of robustness.
This becomes particularly challenging in socially impactful domains, where the objective extends beyond cumulative reward to include fairness, safety, and multi-stakeholder welfare. 
For instance, in autonomous driving and decision-making systems, efficiency may need to be traded off for risk sensitivity and safety. Similarly, in resource allocation and transportation, equitable outcomes across user groups may be required.
Crucially, such safety or fairness preferences can evolve depending on context, regulatory changes, or societal expectations.
Yet, in most deep RL methods, these requirements are addressed during training, and adapting to new preferences or new stakeholders often requires retraining. 
In practice, such retraining is often impractical due to the associated computational cost, time, and data requirements~\citep{mujtaba2025aligning}.
Consequently, there is a need for techniques that can adapt \emph{pretrained} RL policies to evolving fairness objectives without retraining the base policy. For instance,  when a single base policy serves many clients with heterogeneous fairness needs, or when regulatory changes demand compliance before retraining is feasible.

In contrast to the static nature of deployed RL agents, recent progress in large language models (LLMs) suggests that post-training alignment can be feasible~\citep{kumar2025llm}. 
These inference-time alignment techniques aim to steer frozen trained models toward new objectives without updating their parameters. Methods such as reward-guided decoding, representation steering, and lightweight policy adaptation have demonstrated that frozen models can be flexibly directed toward new objectives via auxiliary scoring or controlled decoding strategies~\citep{li2023inference,xu2024genarm,lin2025parm}. Other methods, including RL from human feedback (RLHF)~\citep{ouyang2022training, kaufmann2024survey}, direct preference optimization (DPO)~\citep{rafailov2023direct}, and activation steering~\citep{turner2023steering}, further illustrate how behavior can be reshaped at deployment time through preference models or inference-time control.
These techniques are effective mainly for autoregressive LLMs, where generation can be viewed as a short-horizon decision process with immediate feedback and where past decisions (generated tokens) are automatically carried forward in the model’s context. However, the applicability of this approach in RL is an open problem as the objective, such as fairness, is typically defined over long horizons (e.g., expectation over trajectories) rather than as instantaneous, per-step signals.
Moreover, standard deployed RL policies are Markovian, meaning that they only condition on the current state.
Unlike autoregressive generation, past actions do not necessarily remain available to the policy unless they are encoded into the state. 
Therefore, it remains unclear whether and how inference-time alignment can solve sequential decision-making problems.
This paper addresses precisely this question: 
\textit{\textbf{Given a fixed, pretrained RL policy, can we steer its long-term behavior toward objectives such as fairness at deployment time, while preserving its core performance?}} 

\begin{figure}[t]
    \centering
    \includegraphics[width=1.0\linewidth]{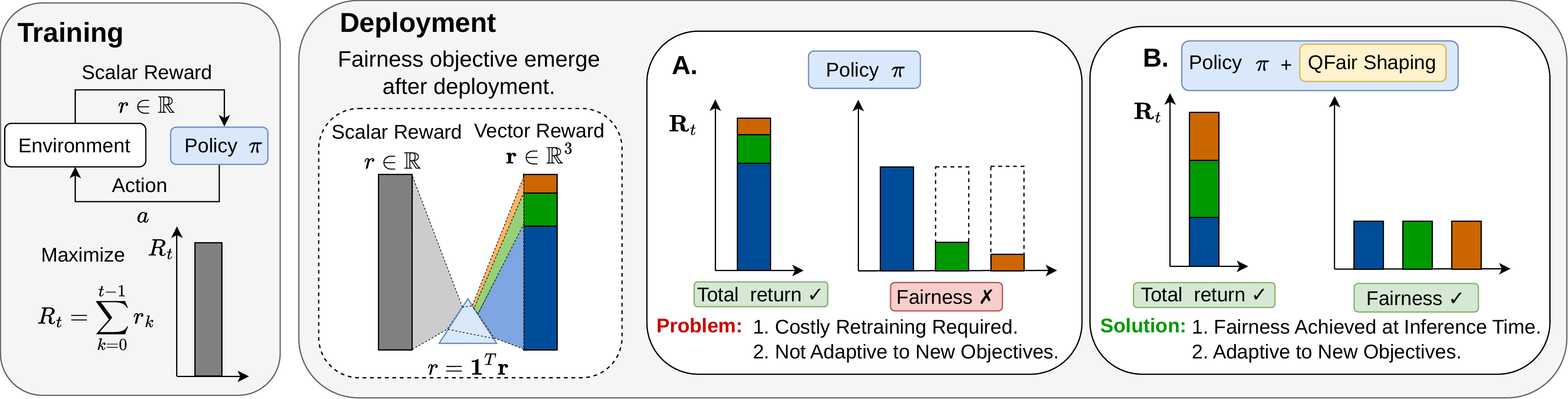}
    \caption{Motivation for inference-time fairness alignment. Left: During training, the policy $\pi$ optimises a scalar return via standard RL. Center: At deployment, the scalar reward is revealed to be structurally decomposable into a vector reward $\mathbf{r} \in \mathbb{R}^n$. Right (A): Deploying the frozen policy $\pi$ directly yields high total return but inequitable allocation across objectives. Right (B): Our approach attaches a lightweight shaping module (QFair) that reweights action probabilities at inference time using welfare-aware Q-values conditioned on the cumulative reward vector. This achieves fair outcomes while preserving the base policy's total return, without retraining the base policy, and with the ability to adapt to new fairness objectives at deployment.}
    \label{fig:motivation}
\end{figure}

Addressing this alignment problem is both practically and ethically important. From a deployment perspective, inference-time alignment reduces retraining costs, improves adaptability, and mitigates training-deployment mismatches~\citep{mujtaba2025aligning}. From a societal perspective, learning human-aligned and fair solutions is essential in high-stakes systems, especially when these systems impact multiple end-users. Existing approaches to aligned and fair RL, including multi-objective RL (MORL) with scalarization~\citep{SiddiqueWengZimmer20, hayes2022practical}, constrained Markov decision processes (CMDPs)~\citep{propfairness_AAMAS20}, and reward shaping~\citep{ng1999policy, kumar2025decaf}, operate at training time and incorporate fairness directly into the learning objective~\citep{jabbari2017fairness, Weng19, siddique2023fairness, siddique2024fairness}. 
While these approaches can produce effective solutions, they generally assume that the desired preferences are known \textit{a priori} and require complete retraining or maintaining a portfolio of policies for different scalarizations~\citep{busa2017multi, zimmer2021learning, do2022optimizing, YuSiddiqueWeng23}. Generalization techniques such as domain randomization~\citep{tobin2017domain} can improve robustness, but they do not enable \emph{post hoc} objective alignment. Meanwhile, policy shaping methods~\citep{knox2010combining,griffith2013policy} and inference-time alignment strategies in LLMs~\citep{ouyang2022training,rafailov2023direct,balashankar2024infalign,xu2024genarm,lin2025parm} demonstrate inference-time behavioral control, yet they do not formalize fairness alignment in RL via policy shaping.
To date, inference-time alignment has largely remained confined to autoregressive generative models rather than long-horizon sequential control with fairness objectives. 

In this paper, we bridge this gap by proposing an \emph{inference-time fairness alignment} approach in RL via policy shaping, enabling a frozen base policy to be adapted toward fairness objectives at deployment, with only a lightweight critic trained offline. We assume access to a base policy $\pi(a \mid s)$ trained to maximize a standard scalar return, and consider settings in which the environment’s feedback can be decomposed into multiple objective components that support a fairness (welfare) formulation.   
Our goal is to construct a shaped policy $\pi'$ that preserves the base policy’s capabilities (i.e., core reward performance) while steering its \emph{long-term} behavior toward fairness.
To achieve this, we introduce a general multiplicative policy-shaping mechanism that reweights action probabilities at inference time using an action-dependent critic. 
We instantiate this critic as a neural network $Q_\phi$, trained offline to approximate a fairness (welfare) action-value signal.
Through extensive experiments across multiple domains, we demonstrate that this inference-time shaping can steer long-term behavior toward fairness objectives while preserving the base policy’s performance.

Our primary contributions are summarized as follows:
\begin{itemize}
\item We formalize the novel problem of inference-time fairness alignment in deep RL as a policy shaping problem~(\Cref{sec:problem_statement}).
\item  We provide theoretical justification that explains when and why inference-time welfare-based shaping can improve fairness objectives without updating the base policy parameters~(\Cref{sec:shaped_policy}).
\item We propose a multiplicative policy-shaping framework that performs test-time action reweighting using a learned fairness critic, enabling adaptation without retraining the base policy~(\Cref{sec:method}).
\item We develop \textbf{QFair}, a history-dependent fairness $Q$-network that conditions on reward vectors to address the non-stationarity inherent in optimizing non-linear welfare functions~(\Cref{sec:qfair}).
\item We provide extensive empirical validation showing that inference-time policy shaping substantially improves welfare while preserving competitive return~(\Cref{sec:experiments}).
\end{itemize}

\section{Related Work}
\paragraph{Inference Time Alignment.}
Inference-time alignment studies how pretrained models can be steered toward new objectives at deployment without updating their core parameters. This line of work has been particularly prominent in LLMs, where frozen models are guided via decoding strategies, auxiliary reward models, or representation-level interventions~\citep{guan2025survey,lu2023inference,scalena2024multi,mujtaba2025aligning}. Motivated by the linear representation hypothesis, \citet{sharkey2025open} studied activation steering, which posits that high-level concepts correspond to directions in latent space; steering is performed by injecting scaled concept vectors into intermediate activations during the forward pass~\citep{siddique2025shifting}. Recent work further optimizes these steering vectors using RL, showing that training low-dimensional intervention vectors while freezing the backbone can rival fine-tuned models on reasoning benchmarks such as GSM8K and MATH~\citep{sinii2025steering}. Complementary approaches include Inference-time Policy Adapters (IPA)~\citep{lu2023inference} and reward-guided decoding methods such as RLHF~\citep{ouyang2022training} and DPO~\citep{rafailov2023direct}, as well as autoregressive reward modeling frameworks like GemARM and PARM~\citep{xu2024genarm,lin2025parm}. Despite their success, these methods are largely confined to autoregressive generation, where alignment can be framed as short-horizon scoring. In contrast, we extend the paradigm of inference-time alignment to RL with trajectory-level fairness objectives.

\paragraph{Policy Shaping and Test-Time Intervention.}
Policy shaping and test-time intervention in RL address how learned policies can be modified without retraining, often through external guidance or auxiliary evaluators. Early work by \citet{griffith2013policy} formalized policy shaping as a mechanism for incorporating human feedback into action selection, while \citet{ferreira2015reinforcement} integrated social signals into dialogue management systems to refine policies online. Preference-based RL further addressed reward misspecification by inferring reward functions from pairwise human feedback~\citep{wirth2016model,christiano2017deep}. Human-in-the-loop corrective strategies such as COACH~\citep{najar2021reinforcement} demonstrated that lightweight action-level corrections can guide exploration without modifying the base objective. More recently, \citet{mujtaba2025aligning} explicitly introduced test-time policy shaping to mitigate unethical behaviors in text-based agents, using classifiers to penalize undesirable actions during deployment. In parallel, token-level steering approaches for LLMs, including SPI~\citep{zhang2025reinforcement} and decoding-time reward combination via Legendre transforms~\citep{shi2024decoding}, highlight the broader applicability of inference-time control. Latent activation editing methods~\citep{das2025latent,sinii2025steering} intervene directly in representation space, whereas policy shaping operates explicitly over the action distribution. Building on this line of work, our framework formulates fairness alignment in RL as a principled inference-time policy-shaping problem for frozen policies under trajectory-level welfare objectives.

\paragraph{Fairness in RL}
Fairness in RL means ensuring equitable outcomes over sequential interactions, where fairness is inherently trajectory-level and temporally coupled. Foundational work in fairness in ML established formal notions such as demographic parity, equalized odds, and individual fairness~\citep{dworkFairnessAwareness2012,zafarParityPreferencebasedNotions2017,AgarwalBeygelzimerDudikLangfordWallach18,SpeicherHeidariGrgicHlacaGummadiSinglaWellerZafar18,zhang2021fairness}. Beyond static prediction, distributive justice perspectives emphasize welfare-based aggregation and equitable resource allocation~\citep{Rawls71,bramsFairDivisionCakeCutting1996,Moulin04,HeidariFerrariGummadiKrause18,malfareNeurips2021}. In RL, \citet{jabbari2017fairness} introduced fairness constraints on state visitation frequencies, while \citet{Jiang2019} proposed decentralized fairness via gossip-based coordination. Actor–critic formulations incorporating fairness utilities were explored by \citet{chen2021bringing}. \citet{liu2018delayed} studied the delayed impact of fair decisions, showing that static fairness criteria can harm protected groups under one-step feedback dynamics, while~\citet{icarte2018using} employ reward machines to expose automaton-structured reward decompositions to the agent. However, these methods either work with static classifiers or encode fixed task structures, rather than operating at the inference-time level and steering a pre-trained policy towards fairness. MORL has provided a natural framework for fairness via scalarization of vector rewards; notably, \citet{SiddiqueWengZimmer20} optimized the generalized Gini welfare function (GGF) in deep RL, satisfying the Pigou–Dalton transfer principle, and subsequent extensions addressed decentralized MARL and broader welfare formulations~\citep{zimmer2021learning,siddique2024fairness,fan2022welfare,do2022optimizing,YuSiddiqueWeng23ECAI,qian2025fair,michailidis2024scalable,siddique2025towards,siddique2026learning}. These approaches, however, assume that fairness preferences are specified during training or require explicitly learning fair policies. In contrast, our work departs from training-time fairness optimization and instead enables fairness alignment at deployment, steering a pretrained reward-maximizing policy toward welfare-based objectives without updating the base policy's parameters.

\section{Preliminary}
\subsection{Reinforcement Learning with Decomposable Rewards}
\label{sec:rl}
We consider finite-horizon Markov decision processes (MDPs) in which the environment feedback admits a natural decomposition across multiple objectives (e.g., stakeholders, resource types, or user groups). Concretely, we consider \emph{multi-objective MDP} (MOMDP) defined by the tuple $\mathcal{M} = (\St, \Ac, P, \bm{r}, T, \gamma)$, where $\St$ is the state space, $\Ac$ is a finite action space, $P(s' \mid s,a)$ is the transition kernel, $T$ is the horizon, and $\gamma \in (0,1]$ is the discount factor. Here, the reward is a vector $\bm r = (r_1, r_2, ..., r_\nO)$, where each component $r_i$ corresponds to objective $i$. The corresponding (discounted) return vector for a trajectory is given by, $\sum_{t=0}^{T-1}\gamma^t \bm{r}(s_t,a_t)$. In many domains, the scalar reward used for training is a simple decomposition of these components. In particular, a common choice is the utilitarian sum, which can be defined as $r(s,a) = h(\bm{r}(s,a)) = \bm{1}^\top \bm{r}(s,a) = \sum_{i=1}^{\nO} r_i(s,a)$, yielding a standard scalar-reward MDP for training. Given a scalar reward $r$, the usual RL objective is to maximize expected discounted return
$J(\pi) = \Expect_\pi \left[ \sum_{t=0}^{T-1} \gamma^{t} r (s_t,a_t) \right]$.
The state-value and action-value functions under $\pi$ are given by,
  $V^\pi(s) = \Expect_\pi\!\left[\sum_{k=t}^{T-1}\gamma^{k-t}\,r_k \;\middle|\; s_t = s\right],~~
  Q^\pi(s,a) = \Expect_\pi\!\left[\sum_{k=t}^{T-1}\gamma^{k-t}\,r_k \;\middle|\; s_t = s,\, a_t = a\right].$
When the state or action spaces are large, deep RL methods approximate these quantities with neural networks. Value-based methods such as DQN~\citep{MnihKavukcuogluSilverRusuVenessBellemareGravesRiedmillerFidjelandOstrovskiPetersenBeattieSadikAntonoglouKingKumaranWierstraLeggHassabis15} estimate $Q^\pi$ and derive policies via an $\epsilon$-greedy approach. Policy-gradient methods directly optimize parameterized policies $\pi_\theta$ via gradient ascent on $J(\pi_\theta)$~\cite{SuttonBarto98}. Actor–critic methods such as A2C~\citep{MnihBadiaMirzaGravesLillicrapHarleySilverKavukcuoglu16} combine both approaches by learning a critic $V_\phi$ to reduce variance in policy-gradient updates, while PPO~\citep{SchulmanWolskiDhariwalRadfordKlimov17} stabilizes training via clipped surrogate objectives. In this paper, these algorithms are used only to obtain pretrained, reward-maximizing base policies; consequently, we adapt these policies at inference time without retraining.

\subsection{Fairness Formulation}
\label{sec:fairness}
To evaluate fairness across objectives, we adopt the welfare-function framework from distributive justice and multi-objective optimization~\citep{Rawls71, Moulin04, SpeicherHeidariGrgicHlacaGummadiSinglaWellerZafar18, SiddiqueWengZimmer20}. A welfare function $\swf: \mathbb{R}^\nO \to \mathbb{R}$ aggregates the return vector into a scalar and provides a measure of social desirability. Following prior work~\citep{Weng19, SiddiqueWengZimmer20, zimmer2021learning}, we require $\swf$ to satisfy three axiomatic properties, namely efficiency, impartiality, and equity.
Linear aggregation functions such as $\bm{1}^\top\bm{u}$ satisfy efficiency and impartiality but generally fail to satisfy equity, motivating non-linear welfare functions. Therefore, in this paper, we focus on the \emph{generalized Gini welfare function} (GGF)~\citep{Weymark81}, defined as
\begin{equation} \label{eq:welfare}
\swf(\bm{u}) = \sum_{i=1}^{\nO} w_i u_i^{\uparrow},
\end{equation}
where $\bm{u}^{\uparrow}$ denotes the sorted utility vector in ascending order and $\bm{w} \in \mathbb{R}_{>0}^{\nO}$ is a strictly decreasing positive weight vector. GGF satisfies all the above properties as it is symmetric, Pareto-monotonic, and Schur-concave. It interpolates between utilitarian and max–min objectives through the choice of weights. While our framework applies to any concave welfare function satisfying the axioms above, we adopt GGF throughout as a canonical choice. When evaluating fairness, an important distinction arises in how the expectation is applied to the GGF. While prior works~\citep{fan2022welfare} optimize the expected welfare, this formulation cannot theoretically ensure fairness due to the non-linearity of the welfare function. Instead, to properly guarantee the fair treatment across objectives, we must optimize the welfare of expected return. Thus, our fairness objective can be expressed as,$ J_\w (\pi) = \swf (\bm J (\pi))$. Since $\swf$ is non-linear, evaluating the welfare of expected returns is inherently a trajectory-level property that cannot be decomposed into per-step scalar rewards. Consequently, a base policy strictly optimized for the step-wise utilitarian objective $J(\pi)$ generally fails to maximize $J_{\mathbf{w}}(\pi)$. This highlights the necessity for inference-time alignment at deployment.

\section{Inference-Time Fairness Alignment}
\label{sec:problem_statement}
We consider the problem of aligning or shaping a pretrained RL policy toward fair outcomes at inference time without updating its learned parameters. The key observation is that, although the base policy was trained with a scalar reward $r=h(\bm r)$, the environment exposes a richer \emph{vector} reward $\bm r \in \mathbb{R}^{\nO}$ that can be leveraged post hoc to optimize a welfare-based fairness objective.

\paragraph{Assumptions.} We require that (i) a frozen pretrained base policy $\pi (a\mid s)$ whose action distribution (or action values) can be queried, (ii) observability of the vector reward $\bm r \in \mathbb{R}^{\nO}$ at deployment, i.e., the scalar training reward is structurally decomposable and its components are exposed by the environment, and (iii) an offline phase in which QFair is trained from base-policy rollouts with light exploration. Hence, no welfare-driven environment interaction and base policy updates are required.

Formally, our inference-time alignment problem can be written as
\begin{equation} \label{eq:main_problem}
   \max_{\pi'} J_{\w}(\pi') = \max_{\pi'} \swf \left(\bm J(\pi') \right), 
\end{equation}
where the expectation is taken with respect to the shaped policy $\pi'$. The goal is therefore to find a policy $\pi'$ that maximizes the welfare objective $J_{\w}(\pi')$ subject to the constraint that $\pi'$ is obtained strictly by reshaping the action distribution of $\pi$ at inference time. However, solving this problem introduces some unique challenges. Since $\swf$ is non-linear, it cannot in general be expressed as a sum of per-step scalar rewards on the original state space. Moreover, the welfare objective depends on the cumulative return vector accumulated along the trajectory, rather than solely on the current environment state. As a result, the welfare-optimal decision at time $t$ depends on the history of rewards obtained so far. Our approach addresses these challenges through a sequence of principled reductions. First, we show that the non-linearity of $\swf$ induces non-Markovianity in the original MDP, because the welfare objective depends on the full trajectory return. Second, we construct a welfare-augmented MDP $\tM$ that restores the Markov property~(\Cref{sec:augmented_mdp}). Third, we formally relate the pretrained base policy to this augmented environment through a policy lifting construction. Finally, within $\tM$, we derive the shaped policy in closed form as the solution to a KL-regularized welfare surrogate that locally lower-bounds the true welfare improvement~(\Cref{sec:shaped_policy}). We then show how this solution can be estimated in practice using a welfare critic learned from base-policy trajectories. The full proofs of our propositions and theorems are provided in Appendix~\ref{app:theo}.

\subsection{Restoring Markovianity via a Welfare-Augmented MDP}
\label{sec:augmented_mdp}
The core difficulty in optimizing $J_{\w}(\pi')$ is that the welfare function $\swf$ is a non-linear function applied to the full trajectory return. Consequently, the objective cannot be decomposed into independent per-step rewards. This means that the welfare-optimal action at time $t$ depends not only on the current state $s_t$ but also on the history of rewards already accrued. For simplicity, we hereafter assume $\gamma=1$ and thus $J$ becomes the cumulative sum of rewards, denoted as $\bm R$. 

\begin{wrapfigure}{r}{0.45\linewidth}
   \centering
   \includegraphics[width=\linewidth]{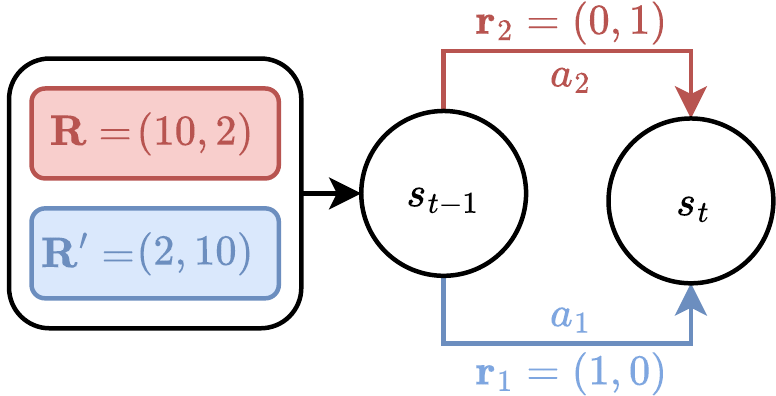}
   \caption{Example where different cumulative rewards lead to different optimal actions.}
   \label{fig:example}
\end{wrapfigure}

\begin{proposition}[Non-Markovianity of Welfare Optimization]
\label{prop:nonmarkov}
When $\swf$ involves sorting (e.g., the GGF defined in~\eqref{eq:welfare}), the welfare-optimal policy is generally not Markovian in $\St$. In particular, there exist states $s$ and distinct reward histories $\bm{R} \neq \bm{R}'$ for which the welfare-optimal actions at $(s, \bm{R})$ and $(s, \bm{R}')$ are different.
\end{proposition}

\paragraph{Example 4.1} \label{ex: non-stationary}
\textit{To illustrate Proposition~\ref{prop:nonmarkov}, we consider an MDP with $\nO = 2$ objectives and GGF weights $w_1 > w_2 > 0$. As seen in~\Cref{fig:example}, at state $s_{t-1}$, action $a_1$ yields $\bm{r}_1 = (1,0)$ and $a_2$ yields $\bm{r}_2 = (0,1)$, both leading to a terminal state $s_t$. If the accrued reward in this trajectory is $\bm{R} = (10,2)$, then $a_2$ produces a final return of $(10,3)$, which sorts $(3,10)$ yielding a welfare score of $3w_1 + 10w_2$. Action $a_1$ produces $(11,2)$, yielding a welfare score of $2w_1 + 11w_2$. Since $w_1 > w_2$ and the difference is $w_1 - w_2 > 0$, $a_2$ is optimal. Conversely, if $\bm{R}' = (2,10)$, the same analysis shows $a_1$ is optimal. Thus, the optimal action at the \emph{same state} $s$ reverses depending on cumulative reward $\bm{R}$.
}

To restore the Markov property, we enrich the state representation with the cumulative reward vector and define a welfare-augmented MDP.
\begin{definition}[Welfare-Augmented MDP]
\label{def:augmented}
Given an MOMDP $\mathcal{M} = (\St, \Ac, P, \bm{r}, T, \gamma)$ and welfare function $\swf$, define the undiscounted finite-horizon MDP $\tM = (\tS, \Ac, \widetilde{P}, \tilde{r}, T)$ where the state space is augmented with the cumulative reward vector, $\tS = \St \times \mathbb{R}^\nO$, providing an augmented state is $\tilde{s}_t = (s_t, \bm{R}_t)$. The transition dynamics update the accumulator deterministically: $ \widetilde{P}\big(\tilde{s}' \mid \tilde{s}, a)\big) = \widetilde{P}((s', \mathbf{R}') \mid (s, \mathbf{R}), a) = P(s' \mid s, a) \cdot \mathbf{1}[\mathbf{R}' = \mathbf{R} + \mathbf{r}(s,a)]$. The scalar reward in $\tM$ is defined as the marginal welfare contribution
    $\tilde{r}\big((s, \bm{R}),\, a\big)
    \;=\;
    \phi_{\bm{w}}\!\big(\bm{R} + \bm{r}(s,a)\big) - \phi_{\bm{w}}(\bm{R}).$
\end{definition}

\begin{theorem}[Augmented MDP Equivalence]
\label{thm:augmented}
$\tM$ is Markovian, and for any policy $\tilde{\pi}$ on $\tS$:
\begin{equation}\label{eq:telescoping}
  \Expect_{\tilde{\pi}}\!\left[\sum_{t=0}^{T-1} \tilde{r}(\tilde{s}_t, a_t)\right]
  \;=\;
  \Expect_{\tilde{\pi}}\!\big[\phi_{\bm{w}}(\bm{R}_T)\big] - \phi_{\bm{w}}(\bm{0}).
\end{equation}
Hence, maximizing the scalar return in $\tM$ is equivalent to maximizing the welfare objective $J_{\w}(\pi')$.
\end{theorem}

The augmented MDP reduces the non-Markovian welfare optimization to a standard scalar-reward MDP.
We define the welfare value functions in $\tM$:
  $\widetilde{V}^{\tilde{\pi}}(\tilde{s})
  = \Expect_{\tilde{\pi}}\!\left[\sum_{k=t}^{T-1}\tilde{r}_k \;\middle|\; \tilde{s}_t = \tilde{s}\right],~~
  \widetilde{Q}^{\tilde{\pi}}(\tilde{s}, a)
  = \tilde{r}(\tilde{s}, a) + \Expect_{s'}[\widetilde{V}^{\tilde{\pi}}(\tilde{s}')],$
and the welfare advantage $\widetilde{A}^{\tilde{\pi}}(\tilde{s}, a) = \widetilde{Q}^{\tilde{\pi}}(\tilde{s}, a) - \widetilde{V}^{\tilde{\pi}}(\tilde{s})$.

\begin{remark}
\Cref{thm:augmented} extends to $\gamma < 1$ by augmenting the state with the discounted cumulative vector $\bm R_t = \sum_{k=0}^{t-1} \gamma^k \bm r(s_k,a_k)$ together with the running discount $\gamma^t$, and defining the marginal welfare reward $\tilde{r}\big((s, \bm{R}),\, a\big) =     \phi_{\bm{w}}\!\big(\bm{R} + \gamma^t \bm{r}(s,a)\big) - \phi_{\bm{w}}(\bm{R})$. The telescoping argument in the proof of~\Cref{thm:augmented} remains unchanged, and the augmented return again equals $\Expect_{\tilde{\pi}}\!\big[\phi_{\bm{w}}(\bm{R}_T)\big] - \phi_{\bm{w}}(\bm{0})$.
\end{remark}

\subsection{Anchoring on the base policy via policy lifting}
Since the base policy $\pi$ is trained in $\mathcal{M}$ with scalar reward, but its shaped policy $\pi'$ must operate in $\tM$, we therefore define its canonical embedding into the augmented state space. Note that both of these MDPs share the same dynamics and action space but differ fundamentally in state representation and objective. To reason about $\pi$ within $\tM$, its canonical embedding is defined as below.

\begin{definition}[Lifting] 
The lifted policy $\tilde\pi$ on $\widetilde{\St}$ is $ \tilde{\pi}(a \mid \tilde{s}) := \pi(a \mid s), \forall \tilde{s} = (s, \bm{R}) \in \tS,\; a \in \Ac$.
\end{definition}
While $\tilde{\pi}$ is generally a suboptimal policy in $\tM$ (because it ignores the $\bm{R}$-component necessary for welfare optimality), it perfectly preserves the base policy's interaction with the environment.

\begin{proposition}[Trajectory Equivalence]
\label{prop:traj_equiv}
Running $\pi$ in original MDP $\mathcal{M}$ and running $\tilde{\pi}$ in augmented MDP $\tM$ induce same distributions over base-state trajectories $\tau = (s_0, a_0, s_1, a_1, \ldots, s_T)$. That is, $\Pr^\pi_{\Md}(\tau) \;=\; \Pr^{\tilde{\pi}}_{\tM}(\tau)$, where $\Pr(\tau)$ is the probability of an entire trajectory $\tau$ occurring.
\end{proposition}

A direct and powerful consequence of Proposition~\ref{prop:traj_equiv} is that rollouts collected by the base policy can be \emph{directly reused} for offline training in $\tM$. We simply annotate each transition with the cumulative reward vector $\bm{R}_t$ (computed post-hoc from the recorded vector rewards) and the marginal welfare reward.
Furthermore, although $\pi$ was trained to maximize $J(\pi)$, it induces a well-defined welfare $J_{\w}(\pi)$ that provides a guaranteed performance baseline. The following result quantifies why a good base policy is a good starting point for welfare improvement.

\begin{proposition}[Welfare Floor]
\label{prop:welfare_floor}
If the training aggregation is the utilitarian sum ($h = \bm{1}^\top$), then for any policy $\pi$:
\begin{equation}\label{eq:welfare_floor}
  \J_{\w}(\pi) \;\ge\; w_\nO \cdot J(\pi),
\end{equation}
where $w_\nO > 0$ is the smallest GGF weight.
\end{proposition}

Essentially, these results, combined with Theorem~\ref{thm:augmented}, provide welfare improvement achievable by any shaped policy $\pi'$:
  $\J_{\w}(\pi') - \J_{\w}(\pi) =
  \widetilde{J}(\tilde{\pi}') - \widetilde{J}(\tilde{\pi}),$
where $\widetilde{J}(\cdot)$ denotes the expected return in $\tM$.
This identity provides a formal bridge, which is \emph{improving fairness over the base policy is mathematically equivalent to a standard policy improvement problem over the lifted policy in $\tM$}.

\subsection{Optimal KL-Regularized Inference-Time Policy Shaping}
\label{sec:shaped_policy}
We now derive the shaped policy as the closed-form solution to a KL-regularized surrogate optimization problem in $\tM$. The KL-regularization serves two purposes. First, it keeps the shaped policy $\pi'$ close to the original policy $\pi$, and second, it ensures the surrogate objective remains a valid approximation to the true welfare gain.
Solving this optimization problem and using the performance difference lemma~\citep{10.5555/645531.656005} applied to $\tM$, we can express the welfare gap as $\widetilde{J}(\tilde{\pi}') - \widetilde{J}(\tilde{\pi}) = T \Expect_{\tilde{s} \sim d^{\tilde{\pi}'}}\!
\Big[\widetilde{A}^{\tilde{\pi}}(\tilde{s}, a)\Big]$. Since $d^{\tilde{\pi}'}$ depends on the unknown policy, we replace it with $d^{\tilde{\pi}}$ to obtain a tractable welfare surrogate that shares the same gradient at $\tilde{\pi}' = \tilde{\pi}$ and provides a local lower bound on the true improvement under a KL penalty~\citep{SchulmanLevineAbbeelJordanMoritz15}. Therefore, to maximize this objective while preserving the environmental competence guaranteed by the welfare floor, we restrict the shaped policy to a KL-divergence trust region around $\tilde{\pi}$:
\begin{align}
  \max_{\tilde{\pi}'} \quad
  & T\;\Expect_{\tilde{s} \sim d^{\tilde{\pi}}}
    \!\Big[\widetilde{A}^{\tilde{\pi}}(\tilde{s}, a)\Big] \label{eq:kl_opt1} \\
  \text{s.t.} \quad
  & \Expect_{\tilde{s} \sim d^{\tilde{\pi}}}
    \Big[\KL\!\big(\tilde{\pi}'(\cdot|\tilde{s})
    \,\big\|\, \tilde{\pi}(\cdot|\tilde{s})\big)\Big]
    \le \varepsilon,
  \qquad
  \tilde{\pi}'(\cdot|\tilde{s}) \in \Delta(\Ac) \;\;\forall\, \tilde{s}, \label{eq:kl_opt2}
\end{align}
We emphasize that problem~\eqref{eq:kl_opt1}-\eqref{eq:kl_opt2} is a surrogate objective as it maximizes a local lower bound on the true welfare improvement, valid within the KL trust region~\citep{SchulmanLevineAbbeelJordanMoritz15}, rather than the true objective~\eqref{eq:main_problem}. The following closed form is therefore optimal for the surrogate, and constitutes one step of KL-regularized policy improvement over the lifted base policy.

\begin{theorem}[Optimal Shaped Policy]
\label{thm:shaped}
The unique solution to~\eqref{eq:kl_opt1}-\eqref{eq:kl_opt2} is:
\begin{equation}\label{eq:shaped_policy}
  \boxed{
    \pi'(a \mid s, \bm{R})
    \;=\;
    \frac{
      \pi(a \mid s)\;\exp\!\Big(\frac{1}{\beta}\,\widetilde{Q}^{\tilde{\pi}}\!\big((s,\bm{R}),\, a\big)\Big)
    }{
      \displaystyle\sum_{a'} \pi(a' \mid s)\;\exp\!\Big(\frac{1}{\beta}\,\widetilde{Q}^{\tilde{\pi}}\!\big((s,\bm{R}),\, a'\big)\Big)
    }
  }
\end{equation}
where $\beta > 0$ is the Lagrange multiplier for the KL constraint, determined by $\varepsilon$. The shaped policy $\pi'$ is computed directly from $\pi$ and the welfare value function $\widetilde{Q}^{\tilde{\pi}}$ with no iterative update.
\end{theorem}


\section{Training Welfare Critic and Inference-Time Shaping}
\label{sec:method}
To solve the problem~\eqref{eq:main_problem} of finding a fair policy at inference, an agent must evaluate the welfare action-value function $\widetilde{Q}^{\tilde{\pi}}(\tilde{s}, a)$ derived in~\Cref{thm:shaped}. Because computing this quantity exactly is intractable in environments with large or continuous state spaces, we introduce \textbf{QFair}, a lightweight neural network welfare critic, parametrized by $Q_\phi$ that learns to approximate $\widetilde{Q}^{\tilde{\pi}}$ from offline data. The design and integration of QFair bridge our theoretical framework developed in~\Cref{sec:problem_statement} and enables practical inference-time fairness alignment with standard deep RL algorithms. Crucially, QFair is the only component in our framework that requires training. Once learned, it is frozen and utilized purely for inference-time scoring.

\subsection{Learning the Welfare Critic (QFair)} \label{sec:qfair}
We train QFair to approximate the welfare Q-function in $\tM$ from offline data, using an architecture that is influenced by the non-Markovianity of the welfare objective (as established in~\Cref{prop:nonmarkov}). A critic conditioned only on the base state $s$ cannot correctly assess welfare contributions, as the same action can yield opposite fairness impacts depending on the cumulative reward. Therefore, $Q_\phi$ is constructed to take the fully augmented state $\tilde{s} = (s, \bm{R})$ as input. The network outputs a welfare score for each action, which approximates $\widetilde{Q}^{\tilde{\pi}}(\tilde{s},a)$. We parameterize $Q_{\phi}$ as a three-layer MLP in which the state and accrued reward are concatenated at the input layer.

To train QFair from base policy rollouts with light exploration and without requiring additional environment interactions, we rely on the trajectory equivalence result established in~\Cref{prop:traj_equiv}. This result shows that the trajectories from the base policy in original MDP $\mathcal{M}$ induce the same state-action distribution as the lifted policy in the augmented MDP $\tM$. Therefore, data collected in the original environment can be reused to train QFair. To collect offline data, we roll out the behavior policy $\mu(a \mid s) = (1 - \beta_{\mathrm{explore}})\;\pi(a \mid s) + \beta_{\mathrm{explore}}\; (1/|\Ac|)$ that mixes the frozen base policy with uniform exploration. Each transition is then annotated post-hoc with the running accumulator $\bm R_t$ and the marginal welfare reward $\tilde{r}_t = \phi_{\bm{w}}(\bm{R}_{t+1}) - \phi_{\bm{w}}(\bm{R}_t)$, yielding a dataset $\sD = \{(\tilde{s}_t, a_t, \tilde{r}_t, \tilde{s}_{t+1}, d_t)\}$ of transitions in $\tM$. Note that light uniform exploration is included only because near-deterministic base policies under-explore the augmented state space, which is insufficient for TD generalization.

Because Theorem~\ref{thm:augmented} guarantees that $\tM$ is a proper MDP, the Bellman equation holds, and we can train QFair using TD learning  $\widetilde{Q}^\mu(\tilde{s}, a) = \tilde{r}(\tilde{s}, a)
  + \gamma\,\Expect_{s' \sim P}\!\Big[\textstyle\sum_{a'}\mu(a'|s')\,\widetilde{Q}^\mu(\tilde{s}', a') \Big]$. 
Then we minimize the mean-squared Bellman error with Expected SARSA targets:
\begin{align}\label{eq:td_target}
  y = \tilde{r} + \gamma(1-d)\textstyle\sum_{a'}\mu(a'|s')\,Q_\phi(\tilde{s}', a'),
  \qquad
  \mathcal{L}(\phi) = \Expect_{(\tilde{s},a,\tilde{r},\tilde{s}',d)\sim\sD}\!\Big[(Q_\phi(\tilde{s}, a) - y)^2\Big].
\end{align}
Expected SARSA averages over the behaviour policy's action distribution at $\tilde{s}'$, avoiding the maximization bias of Q-learning while remaining consistent with the data collection process. At convergence, $Q_\phi$ approximates the welfare Q-function $\widetilde{Q}^\mu$ in $\tM$, which is close to $\widetilde{Q}^{\tilde{\pi}}$ and required for policy shaping. The full algorithm for training QFair is provided in~\Cref{alg:train} (see Appendix~\ref{app:algos}).

\subsection{Inference-time Policy Shaping}
\label{sec:framework}

\begin{figure}[t]
    \centering
    \includegraphics[width=1.0\linewidth]{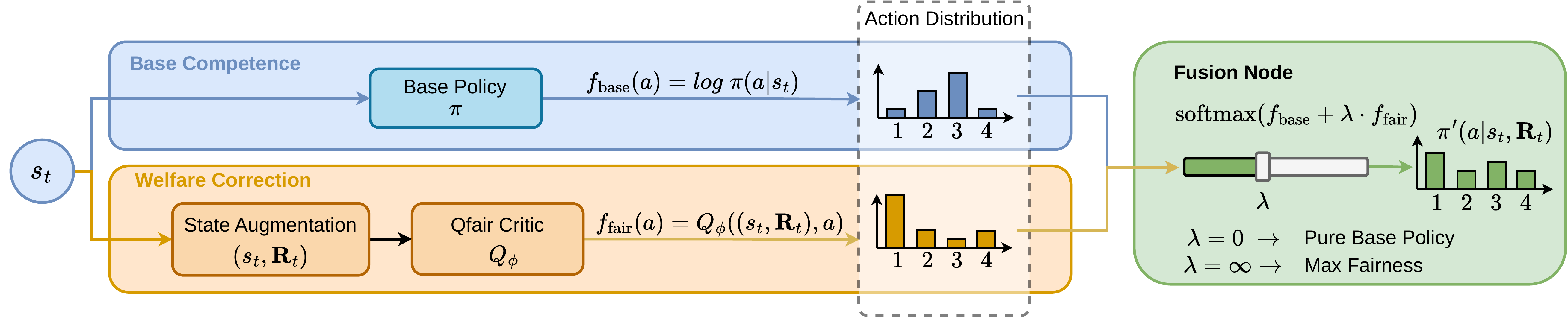}
    \caption{Inference-time policy shaping pipeline. At each step, the current state $s_t$ is processed along two parallel paths. The base policy produces the base competence signal, while the QFair critic receives the augmented state and outputs the welfare correction signal . The fusion node combines both signals with a softmax function, where the alignment strength $\lambda$ interpolates continuously between the pure base policy and maximum fairness. Both the base policy and the QFair critic are frozen and no parameters are updated during deployment.}
    \label{fig:inference}
\end{figure}
With QFair trained offline, the final step is to steer the pretrained policy. At inference time, the agent observes $s_t$ and maintains a running accumulator $\bm{R}_t$ (initialised to $\bm{0}$ at episode start), and computes the shaped action distribution via a softmax combination:
\begin{equation}\label{eq:combined}
  \pi'(a \mid s_t, \bm{R}_t)
  \;=\;
  \mathrm{softmax} \, \Big(f_{\mathrm{base}}(a) + \lambda\, f_{\mathrm{fair}}(a)\Big),
\end{equation}
where the two signal components are $f_{\mathrm{base}}(a)$ which is a standard RL base policy and $f_{\mathrm{fair}}(a) = \bar{Q}_\phi\!\big((s_t, \bm{R}_t),\, a\big)$. To ensure numerical stability and comparable scaling between the two signals, $\bar{Q}_\phi$ denotes the welfare scores normalized to zero-mean and unit-variance across actions.

This softmax formulation mathematically implements the theoretical exponential reweighting derived in Eq.~\eqref{eq:shaped_policy}. Thus, this formulation provides a general framework for policy shaping at test time and is compatible with both value-based and policy gradient methods. In our experiments, we demonstrate the validity of this framework by successfully aligning policies trained with DQN, A2C, and PPO. The full algorithm of inference-time policy shaping is provided in~\Cref{alg:infer}.

\section{Experiments}
\label{sec:experiments}
We evaluate our framework across three domains of increasing complexity and number of objectives. In all environments, fairness is defined at the user (objective) level. During training, standard RL agents optimize the scalar sum of rewards. At inference time, we decompose this scalar reward into its vector components and apply our shaping mechanism to enforce welfare-based fairness. To ensure a comprehensive evaluation, we select domains with varying numbers of objectives and dynamics, making fairness non-trivial. The domains include species conservation, four-room environment, and harvest–regrow. Each domain introduces distinct structural challenges where balancing objectives is essential. Due to space constraints, the environment description and full empirical results for the four-room domain have been deferred to Supplementary Materials~\ref{app:additional_exp}.

\subsection{Setup and Baselines}
We reshape three widely used standard deep RL algorithms, including PPO, A2C, and DQN, covering both policy-based and value-based methods. For each algorithm, we compare (i) standard RL algorithms trained to maximize cumulative scalar reward, (ii) GGF adaptations for PPO, A2C, and DQN that are fairness baselines trained directly to optimize the GGF welfare function following \citet{SiddiqueWengZimmer20}, and (iii) Shaped algorithms for PPO, A2C, and DQN that were trained for reward maximization and aligned toward GGF at inference time via policy shaping. The GGF-based methods serve as true fairness baselines, as they directly optimize the welfare objective during training. In contrast, our shaped agents are trained purely for reward maximization and only steered toward welfare objectives at deployment. All algorithms are evaluated over 20 random seeds in Species Conservation, 10 in Four-Room, and 5 in Harvest–Regrow. Hyperparameters are tuned using Optuna, and full configurations are reported in Appendix~\ref{app:hyper}. 

\begin{figure}[t]
    \centering
    \includegraphics[width=0.8\linewidth]{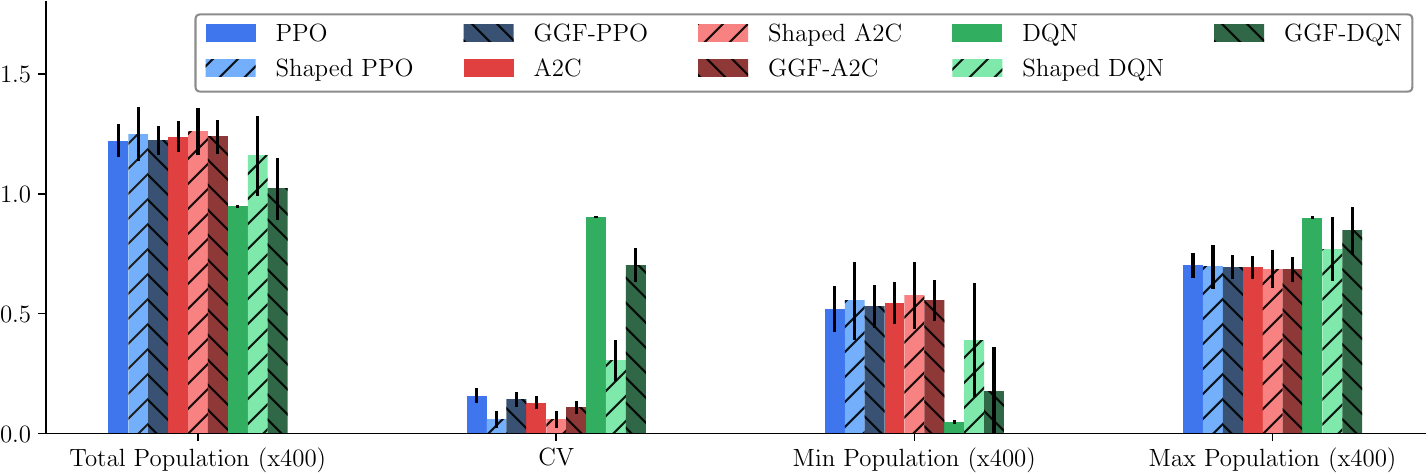}
    \caption{Performance comparison of PPO, A2C, and DQN with their inference-time shaped variants and GGF-trained counterparts in the Species Conservation environment in terms of total reward, CV, min population, and max population }
    \label{fig:sc_bar}
\end{figure}

\subsection{Species Conservation}
Our first domain is a species conservation (SC) environment, which addresses a critical ecological challenge: balancing the populations of two highly interacting endangered species, the sea otter and the northern abalone. Both species are at risk of extinction, requiring sophisticated management strategies to ensure their survival. We adopt the model proposed by~\citep{chades2012setting}, which simulates the predation relationship between the species, where sea otters prey on abalones. 
The state space is composed of the current population sizes of sea otters and northern abalones. The action space includes introducing sea otters, enforcing anti-poaching measures, controlling sea otter populations, implementing a combination of half-antipoaching and half-controlled sea otters, or taking no action. 
The reward function is defined by the population densities of both species, i.e., $\nO=2$. Fairness in this context is interpreted as achieving a balanced distribution of species densities to ensure their preservation. During training, agents optimize the scalar sum of species populations. At inference time, we decompose this sum into its two components and apply policy shaping toward the GGF objective as described in~\Cref{sec:method}.

\Cref{fig:sc_bar} reports total reward, coefficient of variation (CV), minimum population, and maximum population across different standard RL algorithms, GGF-based methods, and our Shaped RL algorithms. Our results show that standard PPO and A2C achieve higher total reward than DQN. As expected, standard RL methods that are trained to maximize reward tend to increase the dominant species population while ignoring the weaker species, resulting in high maximum values and low minimum values. In contrast, GGF-based methods achieve lower CV and higher minimum population levels in comparison to their standard RL algorithms. This means that GGF-based methods achieve more balanced ecological outcomes. Since lower CV corresponds to a more equitable distribution, this confirms that GGF-based training enforces fairness at the expense of some reward optimality. Our shaped agents consistently achieve the lowest CV among all methods while maintaining competitive total reward. In particular, they produce the highest minimum population levels, indicating improved protection of the worst-off species. These results demonstrate that inference-time shaping can yield solutions that are both balanced and performance-preserving.

\Cref{fig:sc_box} presents welfare scores across 100 evaluation trajectories. For each experiment, we compute the empirical average vector return and apply the welfare function to obtain the final score. Our results show that, although GGF-based methods are trained to directly optimize welfare, our shaped agents often achieve comparable or, in some cases, higher welfare scores. This indicates that inference-time shaping effectively approximates welfare optimization without retraining. \Cref{fig:sc_pareto} shows the Pareto trade-off between CV (x-axis) and welfare score (y-axis). Lower CV indicates more balanced distributions. The results show that our shaped RL methods shift the Pareto frontier toward the desirable region of lower inequality and higher welfare. Notably, Shaped-DQN achieves a substantially lower CV than both standard and GGF-trained DQN while maintaining higher welfare. PPO and A2C show similar improvements. These results highlight that inference-time shaping can move pretrained agents toward fairness without modifying learned parameters.

\begin{figure*}[t]
    \centering
    \begin{subfigure}[t]{0.49\linewidth}
        \centering
	     \includegraphics[width=\linewidth]{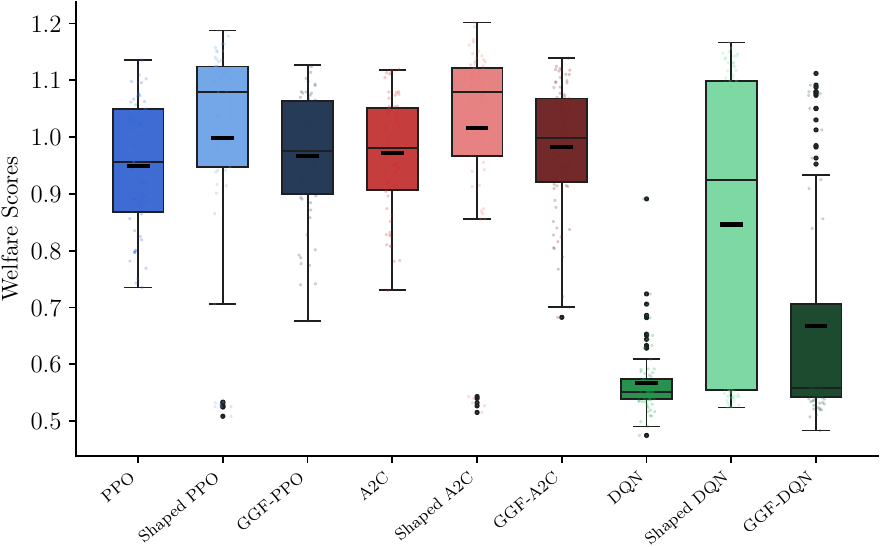}
	      \caption{Welfare scores.}
	      \label{fig:sc_box}
	\end{subfigure}
	\begin{subfigure}[t]{0.49\linewidth}
	    \centering
         \includegraphics[width=\linewidth]{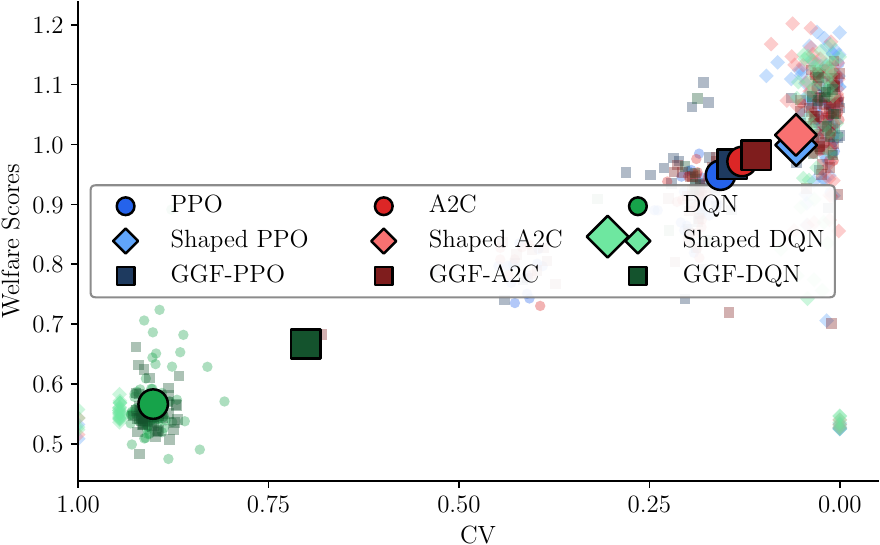}
        \caption{Pareto front plot.}
        \label{fig:sc_pareto}
    \end{subfigure}
    \caption{Performance comparison of PPO, A2C, DQN, their inference-time shaped variants, and GGF-trained counterparts in the Species Conservation.}
    \label{fig:sc}
\end{figure*}

\subsection{Four Room}
Our second domain is a four-room (FR) grid environment, which models a multi-objective navigation and resource-collection problem under spatial constraints and stochastic outcomes. The environment is a $13 \times 13$ maze divided by walls into four connected regions, with multiple spawn locations and three resource types distributed across the map. These resources are represented by distinct shapes and colors: a blue square, a red circle, and a green triangle and the dynamics of these resources are intentionally heterogeneous to create imbalance among objectives. Details of this environment and its empirical results are provided in Supplementary Materials~\ref{app:additional_exp}. 

\subsection{Harvest Regrow}
Our third and final domain is a harvest-regrow (HR) environment, which captures a resource allocation challenge in a spatial ecosystem, where the goal of the agent is to collect resources from multiple renewable but heterogeneous resource types while avoiding strong imbalances across objectives. The environment is a grid of $10 \times 10$ with four crop patches. These crops include apple, melon, berry, and wheat, each with distinct scarcity/value dynamics due to their varying regrowth rates and rewards. In particular, melon yields higher rewards but regrows slowly; wheat provides lower rewards but regenerates quickly; and berry yields are stochastic. In this domain, the agent’s state consists of its current grid position together with four binary harvest indicators specifying whether each crop type has been collected at the current location. The action space comprises the four cardinal movements (up, down, left, right). Rewards are vector-valued with $\nO = 4$, corresponding to per-step crop rewards $[r_{\text{apple}}, r_{\text{melon}}, r_{\text{berry}}, r_{\text{wheat}}]$. Fairness in this setting is interpreted as balanced performance across crop objectives, discouraging policies that over-exploit only the most accessible or highest-value crops. During training, agents optimize the scalarized sum of all crop rewards. At inference time, we decompose the scalar reward into its four components and apply policy shaping toward the GGF objective, as described in~\Cref{sec:method}.

\Cref{tab:hr_results} reports the total reward, CV, minimum crop return, and maximum crop return during the testing phase. The results show that although CV values remain high overall, inference-time shaping consistently reduces inequality relative to the corresponding base algorithms. 
For PPO, A2C, and DQN, the shaped variants achieve lower CV and higher total reward than their standard algorithms. This suggests that inference-time alignment improves fairness without sacrificing efficiency, and in fact yields higher overall returns. Among all methods, GGF-PPO achieves the lowest CV and higher welfare, as expected since it is trained directly to optimize the GGF objective. However, this highest fairness comes at a cost as it has the lowest total reward. In contrast, our shaped methods create a trade-off and learn solutions that are not only fair but also efficient. 

\begin{table}[t]
\centering
\caption{Performance comparison in Harvest-Regrow. 
We report total reward, CV, minimum objective return, and maximum objective return (mean $\pm$ std over evaluation runs).}
\label{tab:hr_results}
\begin{tabular}{lcccc}
\toprule
Algorithm & Total reward & CV & Min crop reward & Max crop reward \\
\midrule
PPO        & 700.32 $\pm$ 9.03   & 0.9997 $\pm$ 0.0003 & 0.01 $\pm$ 0.02 & 700.15 $\pm$ 9.04 \\
Shaped PPO & 804.14 $\pm$ 19.11  & 0.8797 $\pm$ 0.0005 & 0.03 $\pm$ 0.09 & 723.54 $\pm$ 17.25 \\
GGF-PPO    & 483.12 $\pm$ 34.14  & 0.2783 $\pm$ 0.1289 & 52.12 $\pm$ 27.75 & 203.46 $\pm$ 55.35 \\
\midrule
A2C        & 599.80 $\pm$ 155.73 & 0.9995 $\pm$ 0.0004 & 0.01 $\pm$ 0.02 & 599.59 $\pm$ 155.71 \\
Shaped A2C & 694.97 $\pm$ 116.38 & 0.9194 $\pm$ 0.0008 & 0.02 $\pm$ 0.07 & 625.19 $\pm$ 104.67 \\
GGF-A2C    & 613.50 $\pm$ 99.95  & 0.9996 $\pm$ 0.0004 & 0.00 $\pm$ 0.00 & 613.33 $\pm$ 100.04 \\
\midrule
DQN        & 529.89 $\pm$ 30.63  & 0.9461 $\pm$ 0.1115 & 0.09 $\pm$ 0.11 & 505.61 $\pm$ 68.94 \\
Shaped DQN & 638.92 $\pm$ 25.28  & 0.9267 $\pm$ 0.0035 & 0.22 $\pm$ 0.32 & 573.61 $\pm$ 22.83 \\
GGF-DQN    & 618.93 $\pm$ 69.12  & 0.8730 $\pm$ 0.1921 & 0.05 $\pm$ 0.08 & 558.65 $\pm$ 157.50 \\
\bottomrule
\end{tabular}
\end{table}

\section{Conclusions}
In this paper, we address the problem of inference-time fairness alignment in deep RL. By formalizing fairness alignment as a policy shaping problem, we achieve test-time adaptation without retraining the base policy, using our novel history-dependent welfare critic, \textbf{QFair}. We instantiate this multiplicative shaping method across multiple deep RL algorithms and demonstrate through extensive experiments that inference-time policy shaping substantially improves welfare objectives while preserving competitive task performance. In the future, we plan to extend this inference-time shaping framework to accommodate dynamic preferences beyond welfare-based fairness, such as continuous human alignment or interactive human-guided adaptation.

\subsubsection*{Acknowledgments}
\label{sec:ack}
This work was supported by the Office of Naval Research under Grant N000142412405 and the Army Research Office under Grant W911NF2310363.

\appendix





\bibliography{main}
\bibliographystyle{rlj}

\beginSupplementaryMaterials

\section{Theoretical Analysis of Inference-Time Fairness Alignment}
\label{app:theo}
In this appendix, we provide the complete theoretical analysis of our inference-time fairness alignment framework. For clarity and self-containment, we recall the core problem formulation and restate all relevant propositions and theorems from the main text alongside their detailed proofs.

As introduced in Section \ref{sec:problem_statement}, we consider the problem of steering or shaping a pretrained RL policy toward fair outcomes at inference time without updating its learned parameters. The key observation is that, although the base policy was trained with a scalar reward $r=h(\bm r)$, the environment exposes a richer \emph{vector} reward $\bm r\in\mathbb{R}^{\nO}$ that can be leveraged post hoc to optimize a welfare-based fairness objective. Formally, we can write this problem statement as:
\begin{equation} \label{eq:main_problem_app}
   \max_{\pi'} J_{\w}(\pi') = \max_{\pi'} \swf \left(\bm J(\pi') \right), 
\end{equation}
where the expectation is taken with respect to the shaped policy $\pi'$. This means that our goal is to find a shaped policy $\pi'$ that maximizes $J_{\w}(\pi')$ subject to the constraint that $\pi'$ is obtained strictly by reshaping the action distribution of $\pi$ at inference time. However, solving this problem raises some unique challenges. Since $\swf$ is non-linear, it cannot be written as a sum of per-step scalar rewards on the original state space. and the fairness depends on the \emph{cumulative} return vector accrued so far, not only on the current environment state. Our approach solves these challenges through a sequence of principled reductions. First, the core difficulty of non-linearity of $\swf$ is that it is a function of the full trajectory return, which makes it non-decomposable into per-step rewards and non-Markovian in the original state space (\Cref{sec:augmented_mdp_app}).
We address this by constructing an \emph{augmented MDP} $\tM$ that restores the Markov property, then formally connecting the base policy to this augmented setting.
Within $\tM$, we derive the \emph{optimal} shaped policy in closed form as the solution to a KL-constrained welfare maximization problem (\Cref{sec:shaped_policy_app}).
Finally, we describe how this can be estimated via our welfare critic-based approach. 

\subsection{Restoring Markovianity via a Welfare-Augmented MDP}
\label{sec:augmented_mdp_app}
The core difficulty in optimizing $J_{\w}(\pi')$ is that the welfare function $\swf$ is a non-linear function applied to the full trajectory return. Consequently, the objective cannot be decomposed into independent per-step rewards. This means that the welfare-optimal action at time $t$ depends not only on the current state $s_t$ but also on the history of rewards already accrued. For simplicity, next we assume $\gamma=1$ and thus $J$ becomes the cumulative sum of rewards, as denoted as $\bm R$. 

\begin{proposition}[Non-Markovianity of Welfare Optimization]
\label{prop:nonmarkov_app}
When $\swf$ involves sorting (e.g., the GGF~\eqref{eq:welfare}), the welfare-optimal policy is not Markovian in $\St$. There exist states $s$ and distinct reward histories $\bm{R} \neq \bm{R}'$ for which the welfare-optimal actions at $(s, \bm{R})$ and $(s, \bm{R}')$ are different.
\end{proposition}

\begin{example}
\textit{To illustrate Proposition~\ref{prop:nonmarkov_app}, we consider an MDP with $\nO = 2$ objectives and GGF weights $w_1 > w_2 > 0$. As seen in~\Cref{fig:example}, at state $s_{t-1}$, action $a_1$ yields $\bm{r}_1 = (1,0)$ and $a_2$ yields $\bm{r}_2 = (0,1)$, both leading to a terminal state $s_t$. If the accrued reward in this trajectory is $\bm{R} = (10,2)$, then $a_2$ produces a final return of $(10,3)$, which sorts $(3,10)$ yielding a welfare score of $3w_1 + 10w_2$. Action $a_1$ produces $(11,2)$, yielding a welfare score of $2w_1 + 11w_2$. Since $w_1 > w_2$ and the difference is $w_1 - w_2 > 0$, $a_2$ is optimal. Conversely, if $\bm{R}' = (2,10)$, the same analysis shows $a_1$ is optimal. Thus, the optimal action at the \emph{same state} $s$ reverses depending on cumulative reward $\bm{R}$.
}
\end{example}

To restore the Markov property, we enrich the state representation with the cumulative reward vector and define a welfare-augmented MDP.
\begin{definition}[Welfare-Augmented MDP]
\label{def:augmented_app}
Given an MOMDP $\mathcal{M} = (\St, \Ac, P, \bm{r}, T, \gamma)$ and welfare function $\swf$, define the undiscounted finite-horizon MDP $\tM = (\tS, \Ac, \widetilde{P}, \tilde{r}, T)$ where the state space is augmented with the cumulative reward vector, $\tS = \St \times \mathbb{R}^\nO$, providing an augmented state is $\tilde{s}_t = (s_t, \bm{R}_t)$. The transition dynamics update the accumulator deterministically: $ \widetilde{P}\big(\tilde{s}' \mid \tilde{s}, a)\big) = \widetilde{P}((s', \mathbf{R}') \mid (s, \mathbf{R}), a) = P(s' \mid s, a) \cdot \mathbf{1}[\mathbf{R}' = \mathbf{R} + \mathbf{r}(s,a)]$. The scalar reward in $\tM$ is defined as the marginal welfare contribution:
\begin{equation}\label{eq:marginal_reward_app}
    \tilde{r}\big((s, \bm{R}),\, a\big)
    \;=\;
    \phi_{\bm{w}}\!\big(\bm{R} + \bm{r}(s,a)\big) - \phi_{\bm{w}}(\bm{R}).
  \end{equation}
\end{definition}

\begin{theorem}[Augmented MDP Equivalence]
\label{thm:augmented_app}
$\tM$ is Markovian, and for any policy $\tilde{\pi}$ on $\tS$:
\begin{equation}\label{eq:telescoping_app}
  \Expect_{\tilde{\pi}}\!\left[\sum_{t=0}^{T-1} \tilde{r}(\tilde{s}_t, a_t)\right]
  \;=\;
  \Expect_{\tilde{\pi}}\!\big[\phi_{\bm{w}}(\bm{R}_T)\big] - \phi_{\bm{w}}(\bm{0}).
\end{equation}
Hence, maximizing the standard scalar return in $\tM$ is equivalent to maximizing the welfare fairness objective $J_{\w}(\pi')$.
\end{theorem}

\begin{proof}
Markovianity: $\tilde{r}$ and $\widetilde{P}$ depend on the current augmented state $\tilde{s}_t = (s_t, \bm{R}_t)$ and action $a_t$ only.
Equivalence: $\sum_t \tilde{r}(\tilde{s}_t, a_t) = \sum_t [\phi_{\bm{w}}(\bm{R}_{t+1}) - \phi_{\bm{w}}(\bm{R}_t)] = \phi_{\bm{w}}(\bm{R}_T) - \phi_{\bm{w}}(\bm{0})$ by telescoping.
\end{proof}

The augmented MDP reduces the non-Markovian welfare problem to a standard scalar-reward MDP.
We define the welfare value functions in $\tM$:
\begin{equation}\label{eq:welfare_value_app}
  \widetilde{V}^{\tilde{\pi}}(\tilde{s})
  = \Expect_{\tilde{\pi}}\!\left[\sum_{k=t}^{T-1}\tilde{r}_k \;\middle|\; \tilde{s}_t = \tilde{s}\right],
  \quad
  \widetilde{Q}^{\tilde{\pi}}(\tilde{s}, a)
  = \tilde{r}(\tilde{s}, a) + \Expect_{s'}[\widetilde{V}^{\tilde{\pi}}(\tilde{s}')],
\end{equation}
and the welfare advantage $\widetilde{A}^{\tilde{\pi}}(\tilde{s}, a) = \widetilde{Q}^{\tilde{\pi}}(\tilde{s}, a) - \widetilde{V}^{\tilde{\pi}}(\tilde{s})$.

\subsection{Anchoring on the base policy via policy lifting}
Since the base policy $\pi$ is trained in the standard MDP with scalar reward, but its shaped policy $\pi'$ must operate in $\tM$, we therefore define its canonical embedding into the augmented state space. Note that both of these MDPs share the same dynamics and action space but differ fundamentally in state representation and objective. To reason about $\pi$ within $\tM$, its canonical embedding is defined as below.

\begin{definition}[Lifting] 
The lifted policy $\tilde\pi$ on $\widetilde{\St}$ is $ \tilde{\pi}(a \mid \tilde{s}) \;:=\; \pi(a \mid s)$ for all $\tilde{s} = (s, \bm{R}) \in \tS,\; a \in \Ac$.
\end{definition}
While $\tilde{\pi}$ is generally a suboptimal policy in $\tM$ (because it ignores the $\bm{R}$-component necessary for welfare optimality), it perfectly preserves the base policy's interaction with the environment.

\begin{proposition}[Trajectory Equivalence]
\label{prop:traj_equiv_app}
Running $\pi$ in original MDP $\mathcal{M}$ and running $\tilde{\pi}$ in augmented MDP $\tM$ induce same distributions over base-state trajectories $\tau = (s_0, a_0, s_1, a_1, \ldots, s_T)$. That is, $\Pr^\pi_M(\tau) \;=\; \Pr^{\tilde{\pi}}_{\tM}(\tau)$, where $\Pr(\tau)$ is the probability of an entire trajectory $\tau$ occurring.
\end{proposition}

\begin{proof}
Both trajectory probabilities factor as $\rho(s_0)\prod_{t=0}^{T-1}\pi(a_t \mid s_t)\,P(s_{t+1} \mid s_t, a_t)$. The lifting gives $\tilde{\pi}(a_t \mid \tilde{s}_t) = \pi(a_t \mid s_t)$, and $\widetilde{P}$ marginalises to $P(s' \mid s, a)$ since the $\bm{R}$-update is deterministic.
\end{proof}

A direct and powerful consequence of Proposition~\ref{prop:traj_equiv_app} is that rollouts collected by the base policy can be \emph{directly reused} for offline training in $\tM$. We simply annotate each transition with the cumulative reward vector $\bm{R}_t$ (computed post-hoc from the recorded vector rewards) and the marginal welfare reward~\eqref{eq:marginal_reward_app}.
Furthermore, although $\pi$ was trained to maximize $J(\pi)$, it induces a well-defined welfare $J_{\w}(\pi)$ that provides a guaranteed performance baseline. The following result quantifies why a good base policy is a good starting point for welfare improvement.

\begin{proposition}[Welfare Floor]
\label{prop:welfare_floor_app}
If the training aggregation is the utilitarian sum ($h = \bm{1}^\top$), then for any policy $\pi$:
\begin{equation}\label{eq:welfare_floor_app}
  \J_{\w}(\pi) \;\ge\; w_\nO \cdot \bm{J}(\pi),
\end{equation}
where $w_\nO > 0$ is the smallest GGF weight.
\end{proposition}

\begin{proof}
By definition of the GGF, $\swf(\bm{u}) = \sum_i w_i u^\uparrow_i \ge w_\nO \sum_i u^\uparrow_i$,  applying it to $\J_{\w}(\pi) = \swf \left( \bm J(\pi) \right) = \sum_{i=1}^N w_i J_i(\pi)^\uparrow \ge w_N \sum_{i=1}^N J_i(\pi)^\uparrow = w_\nO \bm{J}(\pi)$ yield the result.
\end{proof}

The bound reveals a separation of concerns, which means a trained base policy contributes \emph{environmental competence} (high total return $J(\pi)$) to a welfare floor. What $\pi$ lacks is equity. By combining trajectory equivalence with Theorem~\ref{thm:augmented_app}, we can characterize the welfare improvement achievable by any shaped policy $\pi'$:
\begin{equation}\label{eq:welfare_gap_app}
  \J_{\w}(\pi') - \J_{\w}(\pi) =
  \widetilde{J}(\tilde{\pi}') - \widetilde{J}(\tilde{\pi}),
\end{equation}
where $\widetilde{J}(\cdot)$ denotes the expected return in $\tM$.
This identity provides a formal bridge, which is \emph{improving fairness over the base policy is mathematically equivalent to a standard policy improvement problem over the lifted policy in $\tM$}.

\subsection{Optimal KL-Regularized Inference-Time Policy Shaping}
\label{sec:shaped_policy_app}
We now derive the shaped policy as the optimal solution to a KL-regularized optimization problem in $\tM$. The KL-regularization serves two purposes. First, it keeps the shaped policy $\pi'$ close to the original policy $\pi$, and second, it ensures the surrogate objective remains a valid approximation to the true welfare gain.
Solving this optimization problem and using the performance difference lemma~\citep{10.5555/645531.656005} applied to $\tM$, we can express the welfare gap as $\widetilde{J}(\tilde{\pi}') - \widetilde{J}(\tilde{\pi}) = T \Expect_{\tilde{s} \sim d^{\tilde{\pi}'}}\!
\Big[\widetilde{A}^{\tilde{\pi}}(\tilde{s}, a)\Big]$. Since $d^{\tilde{\pi}'}$ depends on the unknown policy, we replace it with $d^{\tilde{\pi}}$ to obtain a tractable welfare surrogate that shares the same gradient at $\tilde{\pi}' = \tilde{\pi}$ and provides a local lower bound on the true improvement under a KL penalty~\citep{SchulmanLevineAbbeelJordanMoritz15}. Therefore, to maximize this objective while preserving the environmental competence guaranteed by the welfare floor, we restrict the shaped policy to a KL-divergence trust region around $\tilde{\pi}$:
\begin{align}\label{eq:kl_opt_app}
  \max_{\tilde{\pi}'} \quad
  & T\;\Expect_{\tilde{s} \sim d^{\tilde{\pi}}}
    \!\Big[\widetilde{A}^{\tilde{\pi}}(\tilde{s}, a)\Big] \\
  \text{s.t.} \quad
  & \Expect_{\tilde{s} \sim d^{\tilde{\pi}}}
    \Big[\KL\!\big(\tilde{\pi}'(\cdot|\tilde{s})
    \,\big\|\, \tilde{\pi}(\cdot|\tilde{s})\big)\Big]
    \le \varepsilon,
  \qquad
  \tilde{\pi}'(\cdot|\tilde{s}) \in \Delta(\Ac) \;\;\forall\, \tilde{s}.
\end{align}

\begin{theorem}[Optimal Shaped Policy]
\label{thm:shaped_app}
The unique solution to~\eqref{eq:kl_opt_app} is:
\begin{equation}\label{eq:shaped_policy_app}
  \boxed{
    \pi'(a \mid s, \bm{R})
    \;=\;
    \frac{
      \pi(a \mid s)\;\exp\!\Big(\frac{1}{\beta}\,\widetilde{Q}^{\tilde{\pi}}\!\big((s,\bm{R}),\, a\big)\Big)
    }{
      \displaystyle\sum_{a'} \pi(a' \mid s)\;\exp\!\Big(\frac{1}{\beta}\,\widetilde{Q}^{\tilde{\pi}}\!\big((s,\bm{R}),\, a'\big)\Big)
    }
  }
\end{equation}
where $\beta > 0$ is the Lagrange multiplier for the KL constraint, determined by $\varepsilon$. The shaped policy $\pi'$ is computed directly from $\pi$ and the welfare value function $\widetilde{Q}^{\tilde{\pi}}$ with no iterative update.
\end{theorem}

\begin{proof}
The proof proceeds in four parts: (i) form the Lagrangian, (ii) solve pointwise for each $\tilde{s}$, (iii) replace the advantage with the Q-function, and (iv) verify uniqueness and KKT conditions.

\medskip
\noindent\textbf{Part 1: Lagrangian.}

We introduce a dual variable $\alpha \ge 0$ for the KL constraint and per-state multipliers $\zeta(\tilde{s})$ for the normalisation constraints $\sum_a \tilde{\pi}'(a|\tilde{s}) = 1$.
The non-negativity constraints $\tilde{\pi}'(a|\tilde{s}) \ge 0$ will be satisfied automatically by the form of the solution.
The Lagrangian is:
\begin{align}
  \mathcal{L}(\tilde{\pi}', \alpha, \zeta)
  &= \Expect_{\tilde{s} \sim d^{\tilde{\pi}}}\!\left[\sum_a \tilde{\pi}'(a|\tilde{s})\,\widetilde{A}^{\tilde{\pi}}(\tilde{s}, a)\right] \nonumber\\
  &\quad - \alpha\left(\Expect_{\tilde{s} \sim d^{\tilde{\pi}}}\!\left[\sum_a \tilde{\pi}'(a|\tilde{s})\log\frac{\tilde{\pi}'(a|\tilde{s})}{\tilde{\pi}(a|\tilde{s})}\right] - \varepsilon\right) \nonumber\\
  &\quad + \sum_{\tilde{s}} \zeta(\tilde{s})\left(1 - \sum_a \tilde{\pi}'(a|\tilde{s})\right). \label{eq:lagrangian}
\end{align}
Expanding the KL divergence and writing $p(a) = \tilde{\pi}'(a|\tilde{s})$ and $q(a) = \tilde{\pi}(a|\tilde{s}) = \pi(a|s)$ for readability:
\begin{align}
  \mathcal{L} &= \sum_{\tilde{s}} d^{\tilde{\pi}}(\tilde{s})\sum_a p(a)\left[\widetilde{A}^{\tilde{\pi}}(\tilde{s}, a) - \alpha\log\frac{p(a)}{q(a)}\right] + \alpha\varepsilon + \sum_{\tilde{s}}\zeta(\tilde{s})\left(1 - \sum_a p(a)\right). \label{eq:lagrangian_expanded}
\end{align}

\medskip
\noindent\textbf{Part 2: Pointwise optimisation.}

Since the objective and constraints decompose over states (the KL decomposes as $\Expect_{\tilde{s}}[\KL(\tilde{\pi}'(\cdot|\tilde{s}) \| \tilde{\pi}(\cdot|\tilde{s}))]$ and the simplex constraints are per-state), we can solve for $\tilde{\pi}'(\cdot|\tilde{s})$ independently at each $\tilde{s}$.

Fix $\tilde{s}$ and differentiate $\mathcal{L}$ with respect to $p(a) = \tilde{\pi}'(a|\tilde{s})$:
\begin{align}
  \frac{\partial\mathcal{L}}{\partial p(a)}
  &= d^{\tilde{\pi}}(\tilde{s})\left[\widetilde{A}^{\tilde{\pi}}(\tilde{s}, a) - \alpha\left(\log\frac{p(a)}{q(a)} + 1\right)\right] - \zeta(\tilde{s}). \label{eq:derivative}
\end{align}
Setting this to zero:
\begin{align}
  \widetilde{A}^{\tilde{\pi}}(\tilde{s}, a) - \alpha\log\frac{p(a)}{q(a)} - \alpha &= \frac{\zeta(\tilde{s})}{d^{\tilde{\pi}}(\tilde{s})}. \label{eq:stationary}
\end{align}
Solving for $\log p(a)$:
\begin{align}
  \log p(a)
  &= \log q(a) + \frac{1}{\alpha}\widetilde{A}^{\tilde{\pi}}(\tilde{s}, a) - 1 - \frac{\zeta(\tilde{s})}{\alpha\, d^{\tilde{\pi}}(\tilde{s})}. \label{eq:log_p_app}
\end{align}
Denoting $\beta = \alpha$ and collecting the terms that do not depend on $a$ into a single constant:
\begin{equation}\label{eq:c}
  c(\tilde{s}) \;:=\; -1 - \frac{\zeta(\tilde{s})}{\beta\, d^{\tilde{\pi}}(\tilde{s})},
\end{equation}
we get:
\begin{equation}\label{eq:log_solution_app}
  \log\tilde{\pi}'(a|\tilde{s}) = \log\pi(a|s) + \frac{1}{\beta}\widetilde{A}^{\tilde{\pi}}(\tilde{s}, a) + c(\tilde{s}).
\end{equation}
Exponentiating:
\begin{equation}\label{eq:unnorm_app}
  \tilde{\pi}'(a|\tilde{s}) = \pi(a|s)\cdot\exp\!\left(\frac{1}{\beta}\widetilde{A}^{\tilde{\pi}}(\tilde{s}, a)\right)\cdot\exp\!\big(c(\tilde{s})\big).
\end{equation}

\medskip
\noindent\textbf{Part 3: Normalisation and replacing advantage with Q-function.}

The constant $c(\tilde{s})$ is determined by the normalisation constraint $\sum_a \tilde{\pi}'(a|\tilde{s}) = 1$.
Summing~\eqref{eq:unnorm_app} over $a$:
\begin{equation}\label{eq:norm_app}
  1 = \exp\!\big(c(\tilde{s})\big)\sum_a \pi(a|s)\,\exp\!\left(\frac{1}{\beta}\widetilde{A}^{\tilde{\pi}}(\tilde{s}, a)\right),
\end{equation}
so:
\begin{equation}\label{eq:c_value_app}
  \exp\!\big(c(\tilde{s})\big) = \frac{1}{\sum_a \pi(a|s)\,\exp\!\Big(\frac{1}{\beta}\widetilde{A}^{\tilde{\pi}}(\tilde{s}, a)\Big)}.
\end{equation}
Substituting back into~\eqref{eq:unnorm_app}:
\begin{equation}\label{eq:solution_adv_app}
  \tilde{\pi}'(a|\tilde{s})
  = \frac{\pi(a|s)\,\exp\!\Big(\frac{1}{\beta}\widetilde{A}^{\tilde{\pi}}(\tilde{s}, a)\Big)}{\sum_{a'}\pi(a'|s)\,\exp\!\Big(\frac{1}{\beta}\widetilde{A}^{\tilde{\pi}}(\tilde{s}, a')\Big)}.
\end{equation}

Now we replace the advantage $\widetilde{A}^{\tilde{\pi}}$ with the Q-function $\widetilde{Q}^{\tilde{\pi}}$.
Since $\widetilde{A}^{\tilde{\pi}}(\tilde{s}, a) = \widetilde{Q}^{\tilde{\pi}}(\tilde{s}, a) - \widetilde{V}^{\tilde{\pi}}(\tilde{s})$, substituting into~\eqref{eq:solution_adv_app}:
\begin{align}
  \tilde{\pi}'(a|\tilde{s})
  &= \frac{\pi(a|s)\,\exp\!\Big(\frac{1}{\beta}\big[\widetilde{Q}^{\tilde{\pi}}(\tilde{s}, a) - \widetilde{V}^{\tilde{\pi}}(\tilde{s})\big]\Big)}{\sum_{a'}\pi(a'|s)\,\exp\!\Big(\frac{1}{\beta}\big[\widetilde{Q}^{\tilde{\pi}}(\tilde{s}, a') - \widetilde{V}^{\tilde{\pi}}(\tilde{s})\big]\Big)} \nonumber\\[6pt]
  &= \frac{\pi(a|s)\,\exp\!\Big(\frac{1}{\beta}\widetilde{Q}^{\tilde{\pi}}(\tilde{s}, a)\Big)\cdot\exp\!\Big(-\frac{1}{\beta}\widetilde{V}^{\tilde{\pi}}(\tilde{s})\Big)}{\sum_{a'}\pi(a'|s)\,\exp\!\Big(\frac{1}{\beta}\widetilde{Q}^{\tilde{\pi}}(\tilde{s}, a')\Big)\cdot\exp\!\Big(-\frac{1}{\beta}\widetilde{V}^{\tilde{\pi}}(\tilde{s})\Big)} \nonumber\\[6pt]
  &= \frac{\pi(a|s)\,\exp\!\Big(\frac{1}{\beta}\widetilde{Q}^{\tilde{\pi}}(\tilde{s}, a)\Big)}{\sum_{a'}\pi(a'|s)\,\exp\!\Big(\frac{1}{\beta}\widetilde{Q}^{\tilde{\pi}}(\tilde{s}, a')\Big)}, \label{eq:solution_Q_app}
\end{align}
where the last step uses the fact that $\exp\!\big({-}\frac{1}{\beta}\widetilde{V}^{\tilde{\pi}}(\tilde{s})\big)$ is independent of $a$ and cancels between numerator and denominator. Writing $\tilde{s} = (s, \mathbf{R})$ gives exactly~\ref{eq:shaped_policy_app}.

\medskip
\noindent\textbf{Part 4: Uniqueness and KKT verification.}

\emph{Strict concavity.}
The objective is linear in $\tilde{\pi}'$, and the KL divergence $\KL(\tilde{\pi}'\|\tilde{\pi})$ is strictly convex in $\tilde{\pi}'$ (since $x\log x$ is strictly convex).
Therefore, the feasible set $\{\tilde{\pi}' : \KL \le \varepsilon,\; \tilde{\pi}' \in \Delta(\Ac)\}$ is a convex set, and the Lagrangian objective $\mathcal{L}$ is strictly concave in $\tilde{\pi}'$ for any $\alpha > 0$.
This guarantees the solution is unique.

\emph{Non-negativity.}
Since $\pi(a|s) > 0$ for $a \in \mathrm{supp}(\pi)$ and $\exp(\cdot) > 0$, the solution~\eqref{eq:solution_Q_app} satisfies $\tilde{\pi}'(a|\tilde{s}) > 0$ for all $a \in \mathrm{supp}(\pi)$ and $\tilde{\pi}'(a|\tilde{s}) = 0$ for $a \notin \mathrm{supp}(\pi)$.
The non-negativity constraints are automatically satisfied.

\emph{KKT conditions.}
The stationarity condition is satisfied by construction (Part~2).
The primal feasibility ($\sum_a \tilde{\pi}' = 1$, $\KL \le \varepsilon$) is ensured by normalisation and by choosing $\alpha$ to satisfy the KL constraint.
Dual feasibility ($\alpha \ge 0$) holds by assumption.
Complementary slackness ($\alpha(\Expect[\KL] - \varepsilon) = 0$) determines $\alpha$: either $\alpha = 0$ (the unconstrained optimum already satisfies $\KL \le \varepsilon$) or the KL constraint binds (equality holds) and $\alpha > 0$.
In practice, $\alpha > 0$ and the constraint is active.

\emph{Determining $\beta$.}
The parameter $\beta = \alpha > 0$ is implicitly defined by the constraint:
\begin{equation}\label{eq:beta_constraint_app}
  \Expect_{\tilde{s} \sim d^{\tilde{\pi}}}\!\left[\KL\!\big(\tilde{\pi}'_\beta(\cdot|\tilde{s}) \,\big\|\, \tilde{\pi}(\cdot|\tilde{s})\big)\right] = \varepsilon,
\end{equation}
where $\tilde{\pi}'_\beta$ is the solution~\eqref{eq:solution_Q_app} parameterised by $\beta$.
The left-hand side is a continuous, strictly monotonically decreasing function of $\beta$: as $\beta \to \infty$, $\tilde{\pi}' \to \tilde{\pi}$ and $\KL \to 0$; as $\beta \to 0^+$, $\tilde{\pi}'$ concentrates on the welfare-maximising action and $\KL \to \max$.
By the intermediate value theorem, a unique $\beta > 0$ satisfying~\eqref{eq:beta_constraint_app} exists for any $\varepsilon \in (0, \KL_{\max})$.
\end{proof}

The shaped policy~\eqref{eq:shaped_policy_app} has three structural properties that make it well-suited for inference-time alignment:

\begin{enumerate}[leftmargin=1.5em, label=(\roman*)]
  \item \textbf{Support preservation.}
  Since $\exp(\cdot) > 0$, the shaped policy assigns positive probability to exactly those actions in the support of $\pi$: no actions are introduced or eliminated.

  \item \textbf{Smooth interpolation.}
  The parameter $\lambda := 1/\beta$ controls alignment strength. As $\lambda \to 0$, we recover the base policy $\pi$; as $\lambda \to \infty$, $\pi'$ concentrates on the welfare-maximising action within the support of $\pi$. The intermediate regime provides a continuum between performance and fairness.

  \item \textbf{Log-space decomposition.}
  The log-policy decomposes additively:
  \begin{equation}\label{eq:log_decomp_app}
    \log \pi'(a \mid s, \bm{R})
    \;=\;
    \underbrace{\log \pi(a \mid s)}_{\text{base competence}}
    \;+\;
    \underbrace{\lambda\,\widetilde{Q}^{\tilde{\pi}}\big((s,\bm{R}), a\big)}_{\text{welfare correction}}
    \;-\; \log Z(s, \bm{R}),
  \end{equation}
  where $Z$ is the normalisation constant. The welfare Q-function acts as a \emph{logit correction} that tilts the base distribution toward equitable actions.
\end{enumerate}

\section{Algorithms}
\label{app:algos}

\begin{algorithm}[H]
\caption{\textsc{QFair: Offline Welfare Critic Training}}
\label{alg:train}
\begin{algorithmic}[1]
\Require Base policy $\pi$, environment $\mathcal{M}$, welfare $\phi_{\bm{w}}$, exploration $\beta_{\mathrm{explore}}$, episodes $N$, epochs $E$, learning rate $\eta$
\Ensure Trained welfare critic $Q_\phi$
\State Initialise $Q_\phi$ with random parameters $\phi$
\State $\sD \gets \emptyset$ \Comment{Replay buffer for augmented transitions}
\For{episode $= 1, \ldots, N$}
    \State $s_0 \gets \mathcal{M}.\mathrm{reset}()$;\quad $\bm{R}_0 \gets \bm{0}$
    \For{$t = 0, 1, \ldots, T-1$}
        \State $a_t \sim \mu(\cdot \mid s_t)$ \Comment{Behaviour policy}
        \State $s_{t+1}, \bm{r}_t, d_t \gets \mathcal{M}.\mathrm{step}(a_t)$ \Comment{$\bm{r}_t \in \R^\nO$}
        \State $\bm{R}_{t+1} \gets \bm{R}_t + \gamma^t\,\bm{r}_t$
        \State $\tilde{r}_t \gets \phi_{\bm{w}}(\bm{R}_{t+1}) - \phi_{\bm{w}}(\bm{R}_t)$ \Comment{Marginal welfare~\eqref{eq:marginal_reward_app}}
        \State $\sD \gets \sD \cup \{((s_t, \bm{R}_t),\; a_t,\; \tilde{r}_t,\; (s_{t+1}, \bm{R}_{t+1}),\; d_t)\}$
    \EndFor
\EndFor
\For{epoch $= 1, \ldots, E$}
    \For{mini-batch $\mathcal{B} \subset \sD$}
        \State $y_i \gets \tilde{r}_i + \gamma(1-d_i)\sum_{a'}\mu(a'|s'_i)\,Q_\phi(\tilde{s}'_i, a')$ for each $i \in \mathcal{B}$ \Comment{Expected SARSA~\eqref{eq:td_target}}
        \State $\phi \gets \phi - \eta\,\nabla_\phi \tfrac{1}{|\mathcal{B}|}\sum_{i \in \mathcal{B}} (Q_\phi(\tilde{s}_i, a_i) - y_i)^2$
    \EndFor
\EndFor
\end{algorithmic}
\end{algorithm}

\begin{algorithm}[H]
\caption{\textsc{Inference-Time Welfare-Aligned Action Selection}}
\label{alg:infer}
\begin{algorithmic}[1]
\Require Frozen base policy $\pi$, frozen critic $Q_\phi$, alignment strength $\lambda$
\State $s_0 \gets \mathcal{M}.\mathrm{reset}()$;\quad $\bm{R}_0 \gets \bm{0}$
\For{$t = 0, 1, \ldots$ \textbf{until} episode ends}
    \State $f_{\mathrm{base}}(a) \gets$ base policy signal for all $a \in \Ac$ 
    \State $f_{\mathrm{fair}}(a) \gets \bar{Q}_\phi\big((s_t, \bm{R}_t),\, a\big)$ for all $a \in \Ac$ \Comment{Normalised QFair scores}
    \State $\pi'(\cdot) \gets \mathrm{softmax}\big(f_{\mathrm{base}} + \lambda\, f_{\mathrm{fair}}\big)$ \Comment{Shaped policy~\eqref{eq:combined}}
    \State Select $a_t \gets \argmax_a \pi'(a)$ or sample $a_t \sim \pi'$ \Comment{Greedy or stochastic}
    \State $s_{t+1}, \bm{r}_t, d_t \gets \mathcal{M}.\mathrm{step}(a_t)$
    \State $\bm{R}_{t+1} \gets \bm{R}_t + \gamma^t\,\bm{r}_t$ \Comment{Update accumulator}
\EndFor
\end{algorithmic}
\end{algorithm}

\section{Additional Expriments}
\label{app:additional_exp}
In this section, we present the additional experimental results that were not included in the main paper. As mentioned in Section~\ref{sec:experiments}, we present the complete environmental details and empirical results for the four-room domain here.

\subsection{Four Room}
We first provide the details of the four-room environment, which models a multi-objective navigation and resource-collection problem under spatial constraints and stochastic outcomes. The environment is a $13 \times 13$ maze divided by walls into four connected regions, with multiple spawn locations and three resource types distributed across the map. These resources are represented by distinct shapes and colors: a blue square, a red circle, and a green triangle. In this domain, the state space contains agent’s current grid position together with local binary indicators that signal whether the current cell yields a red, green, or blue resource. The action space includes the four cardinal movements (up, down, left, right), where the movement is constrained by walls and map boundaries, so policy quality depends not only on objective prioritization but also on efficient navigation through the maze. The reward is a vector consisting of $\nO = 3$, which corresponds to the collection of each resource type. The resource dynamics are intentionally heterogeneous: green and blue provide stable unit rewards upon collection, whereas red yields stochastic rewards (success or failure), introducing uncertainty and making balanced collection more challenging. In this environment, fairness is interpreted as avoiding collapse to a single easily obtainable objective and instead maintaining balanced performance across all three reward dimensions. During training, agents optimize the scalar sum of objective rewards. At inference time, we decompose this scalar signal into its three components and apply policy shaping toward the GGF objective as described in~\Cref{sec:method}.

\Cref{tab:fr_results} shows the total reward, CV, minimum reward, and maximum reward during the evaluation phase. The results show that although our inference-time shaped methods achieve total rewards comparable to the original RL algorithms, they consistently produce lower CV values than their corresponding base policies, indicating more balanced objective outcomes. As expected, the GGF-based methods, such as GGF-PPO and GGF-A2C, achieve even lower CV values, since they are explicitly optimized for fairness during training. In contrast, our approach performs fairness alignment purely at inference time without retraining the policy. Interestingly, our shaped PPO achieves the highest total reward among all methods, suggesting that optimizing directly for the GGF objective can sometimes lead to overly conservative solutions with reduced overall reward. On the other hand, our inference-time shaping approach maintains high task performance while improving balance across objectives. 

\begin{table}[h]
\centering
\caption{Performance comparison across algorithms. 
We report total reward, coefficient of variation (CV), minimum objective return, and maximum objective return.}
\label{tab:fr_results}
\begin{tabular}{lcccc}
\toprule
Algorithm & Total & CV & Min reward & Max reward \\
\midrule
PPO        & 289.97 $\pm$ 4.65  & 0.8731 $\pm$ 0.1829 & 0.03 $\pm$ 0.04 & 262.56 $\pm$ 43.70 \\
Shaped PPO & 296.49 $\pm$ 0.54  & 0.6394 $\pm$ 0.1331 & 0.22 $\pm$ 0.21 & 208.13 $\pm$ 42.04 \\
GGF-PPO    & 280.14 $\pm$ 6.21  & 0.2939 $\pm$ 0.1203 & 57.05 $\pm$ 14.43 & 145.96 $\pm$ 23.71 \\
\midrule
A2C        & 281.22 $\pm$ 10.41 & 0.6578 $\pm$ 0.1342 & 0.18 $\pm$ 0.14 & 200.50 $\pm$ 41.14 \\
Shaped A2C & 269.15 $\pm$ 40.69 & 0.5654 $\pm$ 0.0795 & 0.33 $\pm$ 0.15 & 169.89 $\pm$ 44.81 \\
GGF-A2C    & 257.21 $\pm$ 12.47 & 0.5095 $\pm$ 0.0860 & 2.79 $\pm$ 10.74 & 146.70 $\pm$ 21.19 \\
\midrule
DQN        & 286.42 $\pm$ 1.58  & 0.7503 $\pm$ 0.1527 & 0.31 $\pm$ 0.10 & 231.21 $\pm$ 38.21 \\
Shaped DQN & 287.43 $\pm$ 11.93 & 0.5844 $\pm$ 0.1094 & 0.25 $\pm$ 0.10 & 184.06 $\pm$ 32.33 \\
GGF-DQN    & 285.77 $\pm$ 1.93  & 0.6730 $\pm$ 0.1180 & 0.24 $\pm$ 0.17 & 213.89 $\pm$ 29.08 \\
\bottomrule
\end{tabular}
\end{table}

\begin{figure*}[h]
    \centering
    \begin{subfigure}[t]{0.49\linewidth}
        \centering
	     \includegraphics[width=\linewidth]{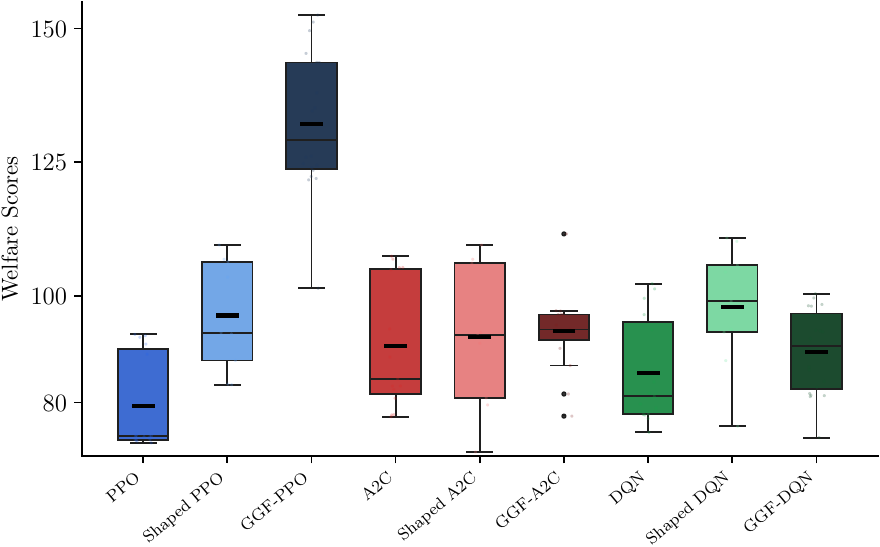}
	      \caption{Welfare scores.}
	      \label{fig:fr_box}
	\end{subfigure}
	\begin{subfigure}[t]{0.49\linewidth}
	    \centering
         \includegraphics[width=\linewidth]{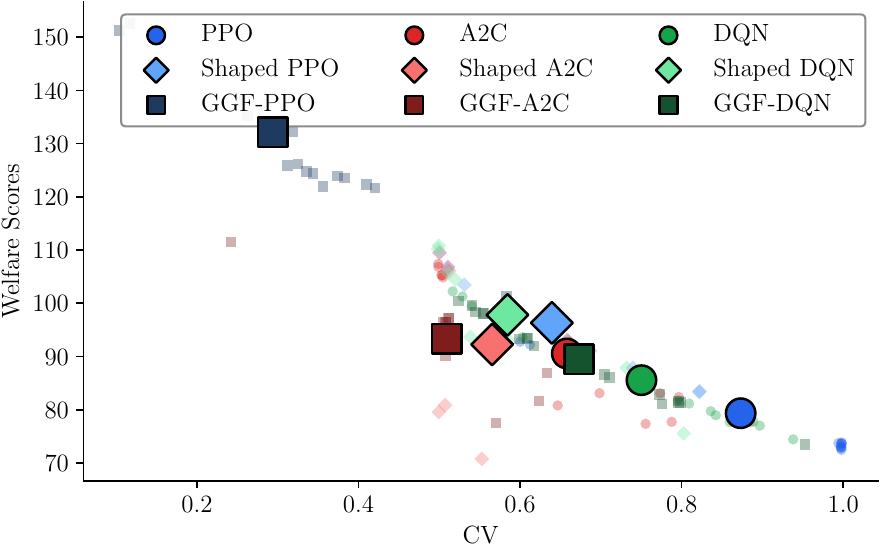}
        \caption{Pareto front plot.}
        \label{fig:fr_pareto}
    \end{subfigure}
    \caption{Performance comparison of PPO, A2C, DQN, their inference-time shaped variants, and GGF-trained counterparts in the Four Room.}
    \label{fig:fr}
\end{figure*}

To further evaluate fairness outcomes, we compute welfare scores using the GGF function applied to the empirical average return vectors. As shown in~\Cref{fig:fr_box}, the inference-time shaped variants consistently achieve higher welfare scores than their standard counterparts. In some cases, such as DQN, the shaped variant even attains higher welfare scores than the GGF-trained model. This result highlights the effectiveness of inference-time alignment for improving fairness without modifying the training objective. We further validate this by showing our results in the Pareto analysis shown in~\Cref{fig:fr_pareto}. These results demonstrate that our inference-time shaping methods shift policies toward the desirable region of lower inequality and higher welfare. While GGF-PPO achieves the strongest fairness performance overall, it is explicitly trained to optimize the GGF objective. On the other hand, our proposed methods consistently improve fairness relative to the original policies.

\subsection{Harvest Regrow}
We also present welfare scores and Pareto front plot of our methods against standard deep RL methods and their GGF counterparts in the harvest regrow environment. \Cref{fig:hr_box} shows welfare values computed from 100 evaluation trajectories per trained agent. Once again, the results show that inference-time shaped variants consistently achieve higher welfare than their standard counterparts and, in some cases, are comparable to GGF-trained policies. Interestingly, Shaped-A2C achieves higher welfare scores than GGF-A2C, even though GGF-A2C is trained to explicitly optimize the GGF objective. As expected, GGF-PPO and GGF-DQN achieve higher welfare scores than their standard and inference-time counterparts. However, our proposed inference-time methods achieve comparable welfare scores to GGF methods while maintaining lower inequality. This is further validated by the Pareto front plot (see~\Cref{fig:hr_pareto}). 
These results suggest that post-training alignment can effectively recover fairness without requiring retraining under a specific fairness objective.

\begin{figure*}[h]
    \centering
    \begin{subfigure}[t]{0.49\linewidth}
        \centering
	     \includegraphics[width=\linewidth]{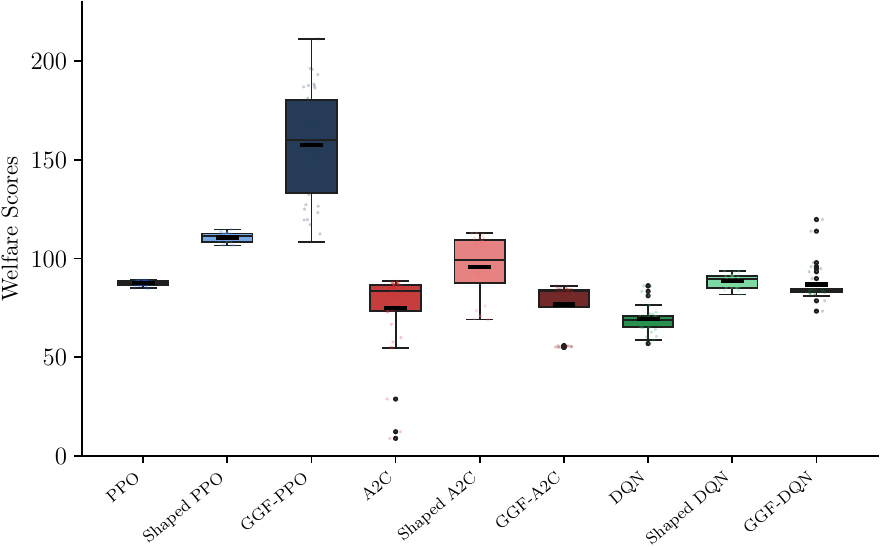}
	      \caption{Welfare scores.}
	      \label{fig:hr_box}
	\end{subfigure}
	\begin{subfigure}[t]{0.49\linewidth}
	    \centering
         \includegraphics[width=\linewidth]{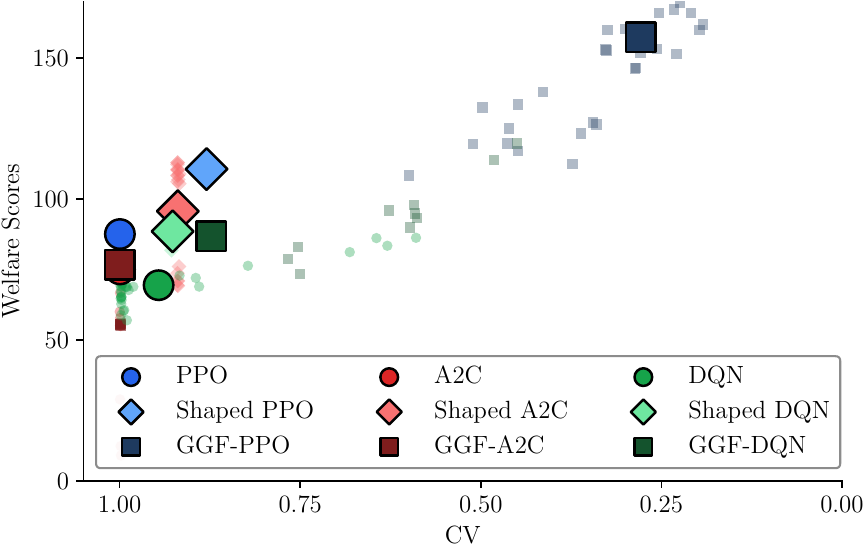}
        \caption{Pareto front plot.}
        \label{fig:hr_pareto}
    \end{subfigure}
    \caption{Performance comparison of PPO, A2C, DQN, their inference-time shaped variants, and GGF-trained counterparts in the Harvest Regrow.}
    \label{fig:hr}
\end{figure*}

\section{Hyperparameters}
\label{app:hyper}

In this section, we detail the hyperparameters and training configurations used across all experiments. For the base RL agents, we utilized the implementations provided by the Stable Baselines3 library~\footnote{https://github.com/DLR-RM/stable-baselines3}. To ensure robust performance and a fair baseline comparison, we tuned the hyperparameters for these agents using the Optuna optimization framework. For the multi-objective baselines that optimize the welfare function directly during training (i.e., GGF-PPO, GGF-A2C, and GGF-DQN), we adopted the author-provided hyperparameters detailed in~\citep{SiddiqueWengZimmer20}, which originally proposed these algorithms. For our inference-time shaped agents (Shaped PPO, Shaped A2C, Shaped DQN), the parameters of the base policies remain completely frozen. The hyperparameters listed for the Shaped agents therefore correspond to the offline training of the QFair critic and the inference-time temperature scaling ($\lambda$). The tables below summarize the hyperparameter values. To denote environment-specific tuning, we use the following subscripts: $sc$ for the species conservation environment, $fr$ for the four-room environment, and $hr$ for the harvest regrow crops environment.

\begin{table}[h]
\centering
\caption{Hyperparameters for PPO and Shaped PPO}
\label{tab:ppo_hyperparams}
\begin{tabular}{lcc}
\toprule
\textbf{Hyperparameter} & \textbf{PPO (Base)} & \textbf{Shaped PPO} \\
\midrule
Discount factor ($\gamma$) & $0.99_{sc, fr}, 0.95_{hr}$ & Base Frozen \\
Learning rate & $3 \times 10^{-4}_{sc, fr, hr}$ & $1 \times 10^{-3}_{sc, fr, hr}$ (QFair) \\
Optimizer & Adam$_{sc, fr, hr}$ & Adam$_{sc, fr, hr}$ (QFair) \\
Network size & [128, 128]$_{sc, fr, hr}$ (Actor/Critic) & [256, 256, 256]$_{sc, fr, hr}$ (QFair) \\
Batch size & $64_{sc, fr}, 128_{hr}$ & $256_{sc, fr, hr}$ (QFair) \\
n\_steps & $2048_{sc, fr, hr}$ & N/A \\
Clip range & $0.2_{sc, fr, hr}$ & N/A \\
Entropy coefficient & $0.0_{sc, fr}, 0.01_{hr}$ & N/A \\
Value function coefficient & $0.5_{sc, fr, hr}$ & N/A \\
Training timesteps & $1 \times 10^6_{sc, fr, hr}$ & $5 \times 10^5_{sc, fr, hr}$ (QFair) \\
\bottomrule
\end{tabular}
\end{table}

\begin{table}[h]
\centering
\caption{Hyperparameters for A2C and Shaped A2C}
\label{tab:a2c_hyperparams}
\begin{tabular}{lcc}
\toprule
\textbf{Hyperparameter} & \textbf{A2C (Base)} & \textbf{Shaped A2C} \\
\midrule
Discount factor ($\gamma$) & $0.99_{sc, fr}, 0.95_{hr}$ & Base Frozen \\
Learning rate & $7 \times 10^{-4}_{sc, fr, hr}$ & $1 \times 10^{-3}_{sc, fr, hr}$ (QFair) \\
Optimizer & RMSprop$_{sc, fr, hr}$ & Adam$_{sc, fr, hr}$ (QFair) \\
Network size & [128, 128]$_{sc, fr, hr}$ (Actor/Critic) & [256, 256, 256]$_{sc, fr, hr}$ (QFair) \\
n\_steps & $5_{sc, fr, hr}$ & N/A \\
Entropy coefficient & $0.0_{sc}, 0.01_{fr, hr}$ & N/A \\
Value function coefficient & $0.5_{sc, fr, hr}$ & N/A \\
Training timesteps & $1 \times 10^6_{sc, fr, hr}$ & $5 \times 10^5_{sc, fr, hr}$ (QFair) \\
\bottomrule
\end{tabular}
\end{table}

\begin{table}[h]
\centering
\caption{Hyperparameters for DQN and Shaped DQN}
\label{tab:dqn_hyperparams}
\begin{tabular}{lcc}
\toprule
\textbf{Hyperparameter} & \textbf{DQN (Base)} & \textbf{Shaped DQN} \\
\midrule
Discount factor ($\gamma$) & $0.99_{sc, fr}, 0.95_{hr}$ & Base Frozen \\
Learning rate & $1 \times 10^{-4}_{sc, fr, hr}$ & $1 \times 10^{-3}_{sc, fr, hr}$ (QFair) \\
Optimizer & Adam$_{sc, fr, hr}$ & Adam$_{sc, fr, hr}$ (QFair) \\
Network size & [64, 64]$_{sc, fr, hr}$ (Q-Net) & [256, 256, 256]$_{sc, fr, hr}$ (QFair) \\
Buffer size & $1 \times 10^5_{sc, fr, hr}$ & $1 \times 10^5_{sc, fr, hr}$ (QFair) \\
Exploration fraction & $0.1_{sc, fr, hr}$ & $\beta_{\text{explore}} = 0.1_{sc, fr, hr}$ \\
Double DQN & True$_{sc, fr, hr}$ & N/A \\
Dueling DQN & True$_{sc, fr, hr}$ & N/A \\
Training timesteps & $1 \times 10^6_{sc, fr, hr}$ & $5 \times 10^5_{sc, fr, hr}$ (QFair) \\
\bottomrule
\end{tabular}
\end{table}

preserving strong task performance.

\end{document}